\RequirePackage{fix-cm}
\documentclass[smallextended]{svjour3}  
\smartqed

\usepackage{graphicx}
\usepackage{enumitem}
\usepackage{amsmath,bbm,amssymb}
\usepackage{breqn}
\usepackage{algorithm}
\usepackage{algorithmic}
\usepackage{bm}
\usepackage{caption}
\usepackage{subfigure}

\newcommand{\tabincell}[2]{\begin{tabular}{@{}#1@{}}#2\end{tabular}}

\newtheorem{prop}{Proposition}
\usepackage{natbib }
\setcitestyle{numbers,square, sort}

\begin{document}

	\title{Use Short Isometric Shapelets to Accelerate Binary Time Series Classification}
%	\title{Short Isometric Shapelet Transform for Binary Time Series Classification}
	
	%\titlerunning{Short form of title}        % if too long for running head
	
	\author{Weibo~Shu \and
		Yaqiang~Yao \and
		Shengfei~Lyu \and
		Jinlong~Li \and
		Huanhuan~Chen         
	}
	
	%\authorrunning{Short form of author list} % if too long for running head
	
	\institute{
		Weibo~Shu \at
		\email{weiboshu@mail.ustc.edu.cn}
		\and Yaqiang~Yao \at 
				\email{yaoyaq@mail.ustc.edu.cn}
		\and Shengfei~Lyu \at
				\email{saintfe@mail.ustc.edu.cn} 
		\and Jinlong~Li \at
				\email{jlli@ustc.edu.cn} 
		\and Huanhuan~Chen \at
			   \email{hchen@ustc.edu.cn} \\
	              University of Science and Technology of China, Hefei, Anhui, China \\
%	              Tel.: +123-45-678910\\
%	              Fax: +123-45-678910\\
	%             \emph{Present address:} of F. Author  %  if needed
}
	
	\date{Received: date / Accepted: date}
	% The correct dates will be entered by the editor

	\maketitle

% As a general rule, do not put math, special symbols or citations
% in the abstract or keywords.
\begin{abstract}
%Recently, Anthony et al. completed a comprehensive review about time series classification(TSC) problems. The result showed that ensemble algorithms such as COTE, ST, EE and so on represented state-of-the-art in accuracy term. But two problems existed, the first one is that the practical time consumption of state-of-the-art is to a certain extent intolerable, which makes barriers for practical promotion of these algorithms in application. The other is that none of those algorithms can prevail on all the data sets. Of course, It is not strange that no algorithm can perform well on all the data sets, but it is necessary to point out what data sets an proposed algorithm can be appropriate to deal with. For partly solving these two problems, we propose a fast new algorithm specific to binary classification problems and accordingly devise a gauging method which judge whether it is appropriate to use our algorithm on some data sets.
In the research area of time series classification, the ensemble
shapelet transform algorithm is one of state-of-the-art
algorithms for classification. 
However, its high time complexity is an issue to hinder its application 
since its base classifier shapelet transform includes a high time complexity of a distance calculation and  shapelet selection. 
Therefore, in this paper we introduce a novel algorithm, i.e. short isometric shapelet transform, 
which contains two  strategies to reduce the time complexity. 
The first strategy of SIST fixes the length of shapelet based on a simplified distance calculation, 
which largely reduces the number of shapelet candidates as well as speeds up the distance calculation in the ensemble shapelet transform algorithm.
The second strategy is to train a single linear
classifier in the feature space instead of an ensemble classifier. 
The theoretical evidences of these two strategies are presented to guarantee a
near-lossless accuracy under some preconditions while reducing the
time complexity. 
Furthermore, empirical experiments demonstrate the superior performance of
the proposed algorithm.
\end{abstract}

\section{Introduction}

Time series are one type of multidimensional data
with a wide range of applications in economy, medical treatments,
and engineering \cite{chen2013learning,gong2018sequential}. The
ordered values in time series contain abundant latent information of
data, which makes time series analysis a challenging task. Time
series classification (TSC) plays a significant role in time series
analysis\cite{chen2013model,yang2017granger}. It appears in a number
of new scenarios such as text
retrieval \cite{Apostolico2002MonotonyOS}, fault diagnosis
\cite{quevedo2014combining,chen2014cognitive}, and bioinformatics
\cite{Gionis2003FindingRS}. 

Many algorithms were proposed to tackle TSC problems during the last decades.
The existing algorithms of TSC can be roughly categorized into the
following five categories \cite{Bagnall2016TheGT}. 
The first category is based on the similarity between two  time
series. Several bespoke distance metrics between pairwise time
series were designed to measure the similarity between  two
time series. Further, the classification relies on these
specific distance metrics
\cite{Mei2016LearningAM,Rakthanmanon2013AddressingBD,Jeong2011WeightedDT,Marteau2009TimeWE,Stefan2013TheMM,Grecki2014NonisometricTI,Batista2013CIDAE,Grecki2012UsingDI,Lines2014TimeSC}.
In contrast to the algorithms based on whole series, the second category 
focuses on extracting local features of time series to classify them
\cite{Roychoudhury2017CostST,Deng2013ATS,Baydogan2013ABF,Baydogan2015TimeSR,Ye2010TimeSS,Keogh2013FastSA,Grabocka2014LearningTS}. 
The third category transforms 
%There is a novel class of algorithms whose main idea is transforming
time series into strings by a uniform dictionary and then measures 
the similarity among these strings
\cite{Lin2007ExperiencingSA,Schfer2017FastAA,Lin2012RotationinvariantSI,Senin2013SAXVSMIT,Schfer2014TheBI}.
The fourth category is model-based algorithms that represent time series with models and then
conduct classification on these models
\cite{Smyth1996ClusteringSW,Corduas2013ClusteringAC,chen2015model,li2019short,gong2018sequential,Bagnall2014ARL}.
The last category  tries to transform source data
into a new feature space and then classify them in the new feature
space
\cite{Kate2015UsingDT,Hills2013ClassificationOT,Bostrom2015BinaryST,Bagnall2015TimeSeriesCW,Lines2016HIVECOTETH,Sharabiani2017EfficientCO}.
Based on the aforementioned algorithms as base classifiers, 
ensemble algorithms \cite{Lines2014TimeSC,Deng2013ATS,Schfer2014TheBI,Lines2016HIVECOTETH} can be built 
to generate a committee for possible better performance \cite{chen2009regularized,chen2010multiobjective,chen2009predictive}, 
since they try to employ accurate yet diverse multiple. 

These various algorithms are often related with each other. 
In the second category of algorithms, shapelets  proposed in \cite{Ye2010TimeSS} are  subseries with informative features, and they can be used 
to build a decision tree where the attribute of each  node is
the distance between time series and a selected shapelet. 
Based on informative shapelets, shapelet transform (ST) constructs a feature space and carries out a classification in this feature space \cite{Hills2013ClassificationOT}.
Finally, the ensemble shapelet transform  algorithm achieves the state-of-the-art in terms of classification 
accuracy by including a number of base classifiers -- shapelet transform \cite{Bagnall2016TheGT,Hills2013ClassificationOT,Bostrom2015BinaryST}.

However,  the high  computational complexity of the shapelet transform  is  severe  to hinder the application of its ensemble algorithms.
On  one hand, in the base classifier, a basic manipulation is calculating the distance between a shapelet and a time series. 
%In most cases, shapelets are not as long as time series. 
The distance between a shapelet denoted by $\mathbf{s}$ and a time series denoted by $\mathbf{x}$ is  calculated by the following formula \cite{Ye2010TimeSS}:
\begin{equation}
	Dist(\mathbf{s},\mathbf{x})=\min \left\lbrace d(\mathbf{s},P) \Big |
	\begin{array}{l}
		P=(x_i,\cdots,x_{i+k-1}), \\
		i=1,\cdots,m-k+1
	\end{array}
	\right\rbrace
	\label{EQ 1}
\end{equation}
where $d(\cdot,\cdot)$ is the function of Euclidean Distance, $P$ is a continuous subsequence, $m$ is the length of time series $\mathbf{x}$, and $k$ is the length of shapelet $\mathbf{s}$. This formula also serves as the distance calculation formula between shapelet candidates  and time series. And the time complexity of a distance calculation is $\mathcal{O}(k(m-k))$, a quadratic item of the length of shapelet.
On the other hand, shapelets are selected from shapelets candidates which are all the subseries of all the time series. 
For a $n$-size and $m$-length time series set, there are $\frac{nm(1+m)}{2} $  candidates in the shapelet transform algorithm.

%However, this algorithm often includes more than necessary
%members, and then some redundancy causes a high time complexity.
%This is a severe issue to hinder its application.
%Some variants have been proposed to tackle this challenge. But most 
%of them might downgrade the accuracy while reducing the time complexity.

To tackle this issue, 
 the short isometric shapelet transform (SIST) algorithm is
proposed to reduce the time complexity of the ensemble shapelet transform algorithm in this paper.
First, SIST fixes the length of shapelet based on a simplified distance calculation, 
which largely reduces the number of shapelet candidates as well as speeds up the distance calculation in the ensemble shapelet transform algorithm.
In addition, SIST substitutes the ensemble classifier with a single classifier in the shapelet-based feature space,
which will sharply decrease training time of the classifier. 
The theoretical evidences are given to demonstrate that the two strategies can guarantee the reduction of time complexity with the near-lossless
accuracy. 
Furthermore, SIST empirically achieves a promising accuracy on most of binary data
sets, which shows its effectiveness.

The rest of the paper is organized as follows. The background and
outline of SIST are introduced in Section
\ref{section:background}. Theoretical analysis about the proposed
strategies in SIST is presented in Section \ref{section:theory}. Section
\ref{section:experiment} shows empirical results of the proposed algorithm
and Section \ref{section:conclusion} concludes this paper.

\section{Short Isometric Shapelet Transform}
\label{section:background}
\subsection{Time Shapelet}
\label{subsection:2-1}
Firstly, two basic terminologies are given as follows:
\begin{description}[style=nextline]
 \item[time series data] One sample of  time series  consists of an  ordered  multi-dimensional value and  a corresponding class label. It is normally expressed with $(\mathbf{x},y)$, where $\mathbf{x}$ is an vector  and $y$ is the class label. The whole time series are denoted by $(\mathbf{X},\mathbf{Y})=((\mathbf{x}_1,y_1),(\mathbf{x}_2,y_2),\cdots,(\mathbf{x}_n,y_n))$. In general, time series in a data set are isometric after the preliminary processing.
 \item[shapelet] The shapelet is the discriminative subseries derived from time series \cite{Ye2010TimeSS}. Subseries are continuous subsequences in time series. A shapelet is expressed with $(\mathbf{s},z)$, where $\mathbf{s}$ is the subseries and $z$ is the class label from its original time series. In a shapelet-based algorithm, each shapelet usually represents a local feature.
\end{description}

Shapelets are selected from the candidates that comprise every 
subseries of all the time series. 
%For a shapelet candidate, distances from it to each time series of a target set are calculated. 
There are four steps to choose informative shaplets. 
In the first step, for each shapelet candidate, 
distances from it to all the time series are calculated. 
After that, a prime segmentation is explored on the real number axis where all the calculated distances are marked. 
Then by classifying time series according to the divided segments which their corresponding distances fall in, 
time series are divided into some disjoint subsets. 
Subsequently, the information gain is calculated according to the division. 
Finally, shapelet candidates with better information gain are selected as shapelets.

In the first step,  
pairwise distances are calculated among the candidates.
Then, a prime segmentation is explored on the real
number axis where all the calculated distances are marked in the second step.
Some disjoint subsets of the candidates are obtained by segmenting them in the third step.
%Then by classifying time series
%according to the divided segments which their corresponding
%distances fall in, time series in the target set are divided into
%some disjoint subsets. 
Finally, the information gain is
calculated according to these disjoint subsets and those with better information gain are selected as shapelets.

Due to the substantial time spent in selecting shapelets
by brute force, Ye and Keogh \cite{Ye2010TimeSS} exploited some tricks, for example early
abandon, to reduce the time consumption.
Several methods  also tried  to solve this time consumption
problem
\cite{Keogh2013FastSA,Grabocka2014LearningTS,Hou2016EfficientLO}.
Although to a certain extent time complexity has been reduced, there
is some loss of accuracy.

\subsection{Shapelet Transform (ST)}
\label{subsection:2-3}

It uses a disposable selection to select $k$ shapelets from all the shapelet candidates. Then it transforms original time series into a $k$-dimensional vector where the value of the $i$th dimension is the distance between the $i$th selected shapelet and time series. 

This method exploits shapelets in a different way from constructing a decision tree. These shapelets are selected only once in this method, whereas the frequency of the shapelets selection equals the number of branch nodes of the decision tree in the algorithm of Ye and Keogh \cite{Ye2010TimeSS}. 
Hence, after this work, the key point of cutting time consumption are turning to selecting the top $k$ most discriminative shapelet candidates as fast as possible \cite{Guido2018FusingTF,Bostrom2016EvaluatingIT}. 
However, the time complexity is still high as a result of the large number of shapelet candidates as well as the complex distance calculation.

%To reduce the high time complexity, the short isometric shapelet transform (SIST) model is proposed in this paper and it will be introduced next.

\subsection{Short Isometric Shapelet Transform (SIST)}
\label{subsection:approach}
For clearly elucidating improvements by a comparison,  SIST  will be introduced with the ensemble ST algorithm (Fig. \ref{Fig 2}).

For a TSC problem, the first step of the ensemble ST algorithm is to extract all the subseries of each time series as shapelet candidates. But in  SIST, only the subseries with a small fixed length will be extracted from each time series. Therefore, the number of shapelet candidates in  SIST  becomes fairly small compared with it in the ensemble ST algorithm. The validity of this strategy will be elaborated later. Step 1 of Fig. \ref{Fig 2} illustrates the difference.

\begin{figure*}[!t]
\centering
\includegraphics[width=\textwidth]{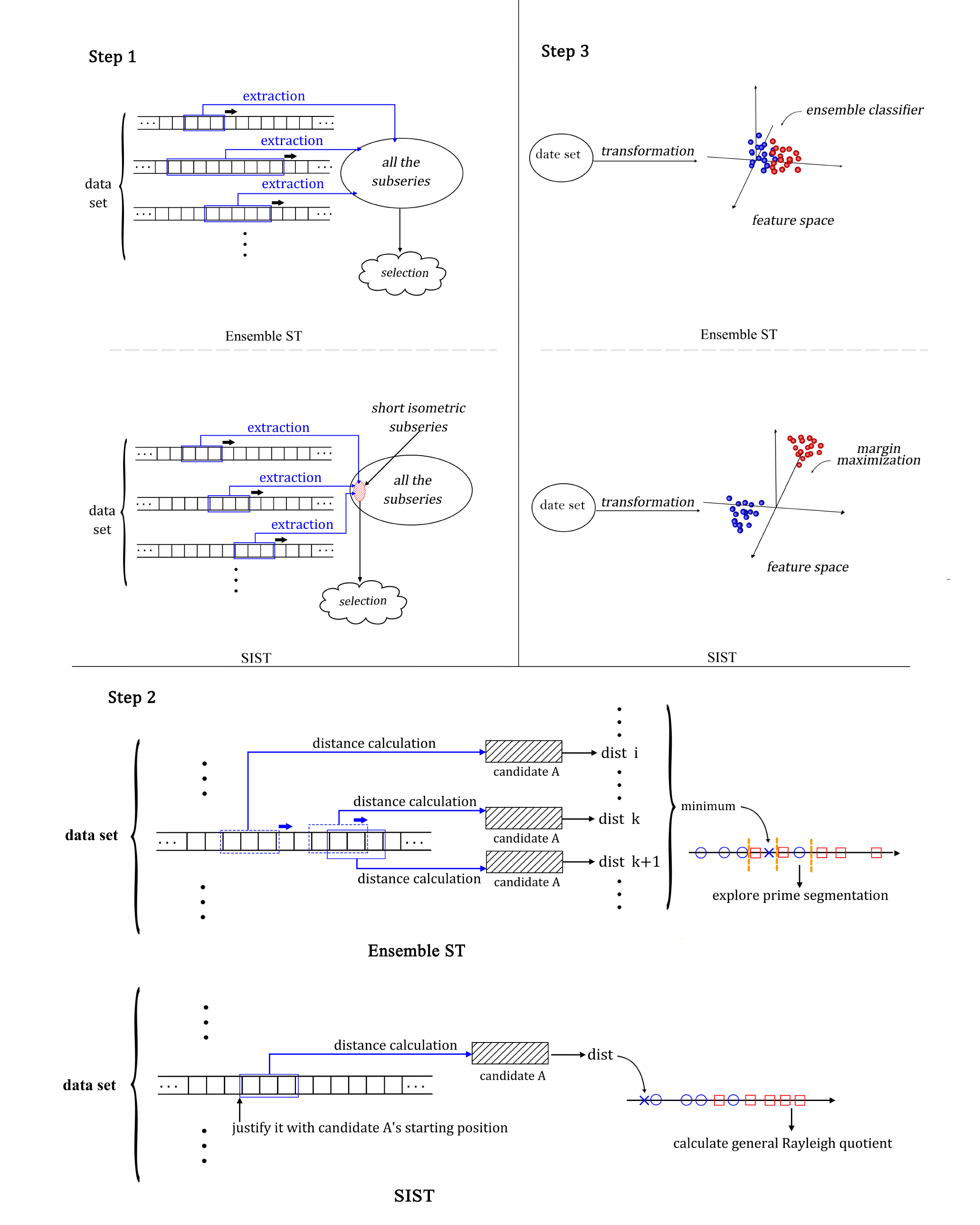}
\caption{Differences between the ensemble ST and  SIST. Details of this figure are discussed in Subsection \ref{subsection:approach}}
\label{Fig 2}
\end{figure*}

The second step is to select shapelets from shapelet candidates. According to the elaboration in Subsection \ref{subsection:2-1}, a process of exploring the prime segmentation in the real axis is indispensable in this step. But in  SIST, this time-consuming process is removed by using the generalized Rayleigh quotient rather than the prime information gain as the discrimination index of shapelet candidates. Besides, the position of shapelet candidates in its original time series is exploited to simplify the distance calculation. Moreover, these improvements will be justified later. And all these differences are illustrated in step 2 of Fig. \ref{Fig 2}.

In the ensemble ST algorithm, the final step is to transform source time series into vectors by selected shapelets (the transformation is described in \ref{subsection:2-3}) and then train an ensemble classifier in that vector space. But we claim that if shapelets are selected by using the generalized Rayleigh quotient of shapelet candidates as the selection priority, there is a high degree of  the linear separability of the transformed data in the vector space. Consequently, only a single SVM is trained to classify data in  SIST. Theoretical evidences to support this claim will be given later. Step 3 of Fig. \ref{Fig 2} shows the difference in the final step. Furthermore, since shapelets are short and distance calculation is simplified in  SIST, much time is saved in transforming time series into the feature space.

\section{Theoretical Analysis}
\label{section:theory}
%In this section, theoretical evidences of the proposed strategies in SIST will be elaborated. We reiterate the two strategies in SIST here. The first one is to  keep shapelet candidates with a small fixed length. This strategy will largely reduce the number of shaplet candidates. The second one is to replace the ensemble classifier with a single SVM in the feature space, which will sharply cut the time consumption of the classifier training.
%The theoretical evidences in this section can guarantee as small as possible loss of accuracy compared with largely reduced time complexity.
%
%The contents of this section are organized as follows. Firstly, the shapelet candidate extraction based on the proposed `Fixed Distance' will be analyzed in Subsection \ref{subsection:3-1} and \ref{subsection:3-2}. This part is about the theoretical evidences of the first proposed strategy. Subsequently, the rationality of substituting the ensemble classifier with a single linear classifier will be discussed in \ref{subsection:3-3}. Finally, the algorithmic flow with the time complexity analysis is given in Subsection \ref{subsection:3-4}.
In this section, the theoretical evidence of the two proposed strategies in  SIST will be elaborated. 
It can guarantee as small loss of accuracy as possible and reduce significantly time complexity. 
The content of this section is organized as follows. 
Firstly, the shapelet candidate extraction based on the proposed `Fixed Distance’ will be analyzed in Subsection \ref{subsection:3-1} and \ref{subsection:3-2}. 
This part explains the feasibility of the first strategy that uses only the shapelet candidates with a small fixed length. 
This strategy will save time by largely reducing the number of shaplet candidates. 
Subsequently, the theoretical supports of the second strategy that substitutes the ensemble classifier with a single linear classifier, 
will be discussed in Subsection \ref{subsection:3-3}. 
This strategy will sharply cut the time consumption of the classifier training in the feature space. 
Finally, the algorithmic flow with the time complexity analysis is given in Subsection \ref{subsection:3-4}.

\subsection{Simplification of the Distance Calculation}
\label{subsection:3-1}
In typical shapelet-based algorithms, it is time-consuming to measure the distance between shapelets and time series. The high time complexity of one distance calculation is drastically amplified by the frequency of the distance calculation.

Hence a basic idea in this paper is to use the position information of shapelets to simplify the distance calculation shown in Eq. (\ref{EQ 1}). Each shapelet or shapelet candidate is a subseries derived from an original time series in the data set. Therefore, its start point is a specific position in its original time series. The starting position can be used to simplify the distance calculation as the following `Fixed Distance' shows.

\begin{equation}
Fixed\_Dist(\mathbf{s},\mathbf{x})= \left\lbrace d(\mathbf{s},P)  \Big |
\begin{array}{c}
P=(x_i,\cdots,x_{i+k-1}),\\
i=j
\end{array}
\right\rbrace
\label{EQ 2}
\end{equation}
where $j$ is the starting position of shapelet $\mathbf{s}$ and other parameters are same with them in Eq. (\ref{EQ 1}). It also serves as the distance calculation formula of shapelet candidates and time series. 
Based on Eqs. (\ref{EQ 1}-\ref{EQ 2}), the time complexity of a distance calculation has been changed from $\mathcal{O}(k(m-k))$ to $\mathcal{O}(k)$.

Both the information the local feature brings and the position where the local feature appears are focused on Fixed Distance in Eq. (\ref{EQ 2}). 
Instead, only the former is concentrated on the  typical distance metric in Eq. (\ref{EQ 1}).
Moreover, if distances between shapelets and time series are calculated by Eq. (\ref{EQ 2}), much time is saved. One part of the saved time is evident from the comparison between Eq. (\ref{EQ 1}) and Eq. (\ref{EQ 2}). The other part is indirect, and details will be elaborated in Subsection \ref{subsection:3-2}. In that subsection, Fixed Distance plays an important role.

There is a simple example illustrating that Eq. (\ref{EQ 2}) reinforces the discrimination ability of shapelets. Data $A$ and $B$ are shown in Fig. \ref{Fig 3}. If the first half of Data $A$ is selected as a shapelet $S$, then by Eq. (\ref{EQ 1}), the distance between $S$ and $A$ is $0$ and the distance between $S$ and $B$ is also $0$. However, if distances are calculated by Eq. (\ref{EQ 2}), then the distance between $S$ and $A$ is $0$ and the distance between $S$ and $B$ is $\sqrt{2}$. Hence, $S$ can classify $A$ and $B$ by Eq. (\ref{EQ 2}) whereas it can not do that by Eq. (\ref{EQ 1}). That is because in Eq. (\ref{EQ 2}), the distance calculation takes into account the position of the local feature represented by the shapelet. In Eq. (\ref{EQ 2}), the distance calculation is fixed at the specific position marked by the start point of the shapelet. That's why it is called `Fixed Distance'.

\begin{figure}[!t]
\centering
\includegraphics[width=4.5in]{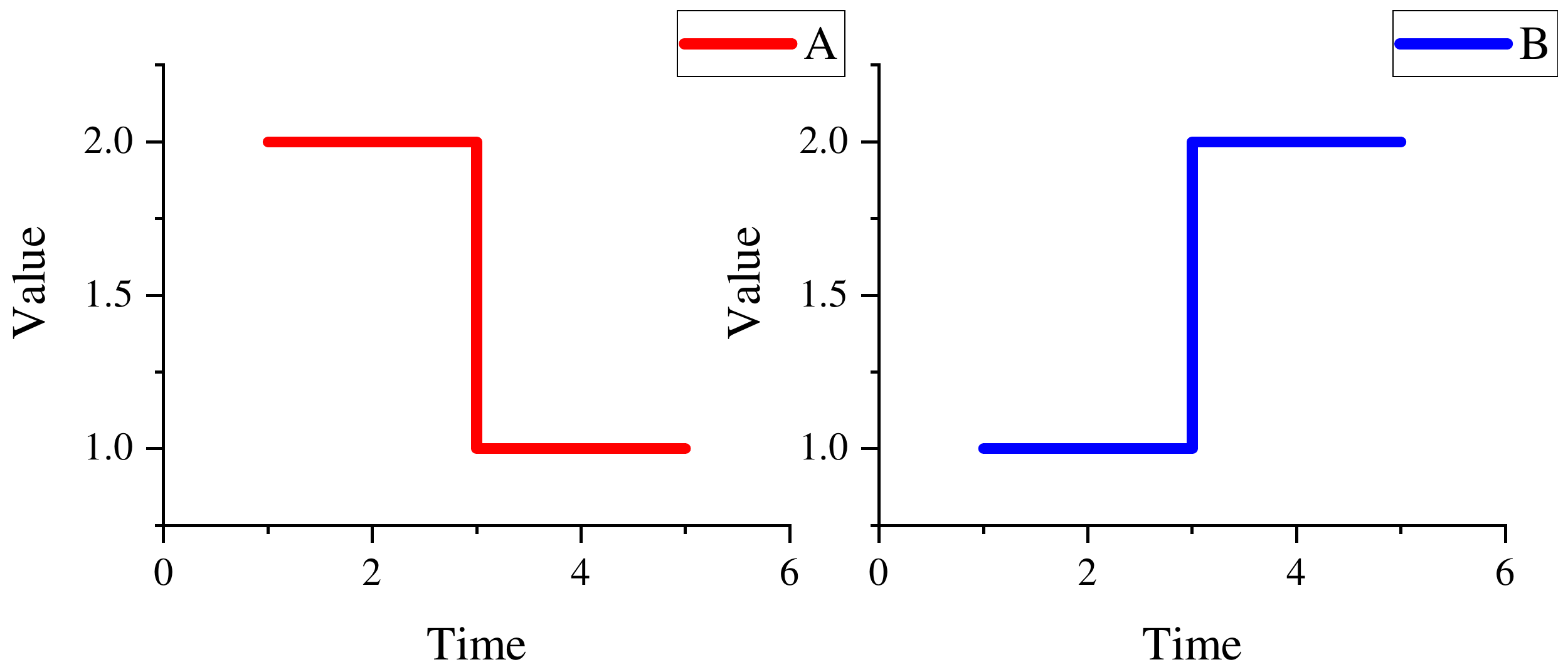}
\caption{Example of the usage of position information. The first half of class A can distinguish A and B by Eq. (\ref{EQ 2}), whereas it can not do that by Eq. (\ref{EQ 1}).}
\label{Fig 3}
\end{figure}

\subsection{Short Isometric Shapelet Candidates}
\label{subsection:3-2}
In this subsection, theoretical analysis focuses on fixing the length of shapelet candidates, namely the strategy showed in Step 1 of Fig. \ref{Fig 2}.

In the process of selecting appropriate shapelets, the time of evaluating a single shapelet candidate is multiplied by the number of all the shapelet candidates. And the number of shapelet candidates is so substantial that the time consumption is huge. If it can be known about the prime length or even a small range of the prime length of shapelets, then shapelet candidates whose length are not the prime one can be discarded. And the huge time consumption vanishes as a result of the drastic reduction of the number of shapelet candidates. Generally speaking, it is hard to know the prime length, or the time spent in searching the best length is not less than the time it saves. However, in  SIST, under the precondition that Fixed Distance is used as the distance calculation formula, the length of shapelet candidates can actually be fixed as a small number  without an appreciation of the prime length. To explain this conclusion, firstly some conceptions should be strictly defined as follows:
~\\

\begin{definition}
\emph{Given a $n$-size set $A$ of time series and a $k$-size set $B$ of shapelets, the process of transforming each time series in A into a $k$-dimensional vector by metric $d$, which means the value at the $i$th dimension of the $k$-dimensional vector is the $d$ metric from the time series to the $i$th shapelet in $B$, is called \textbf{$A$'s shapelet transform through $B$ by metric $d$}.}
\label{defi 1}
\end{definition}
~\\

\begin{definition}
\emph{Given a shapelet $(\mathbf{s},z)$, cut it in arbitrary $n-1$ cut points to produce $n$ ($n>1$) subseries of $(\mathbf{s},z)$. Each subseries is a new shapelet derived from $(\mathbf{s},z)$ and all of them constitute $(\mathbf{s},z)$. The set of the $n$ new shapelets is called the \textbf{$n$-cut set} of $(\mathbf{s},z)$.}
\end{definition}
~\\

\begin{figure}[!t]
	\centering
	\includegraphics[width=0.5\linewidth]{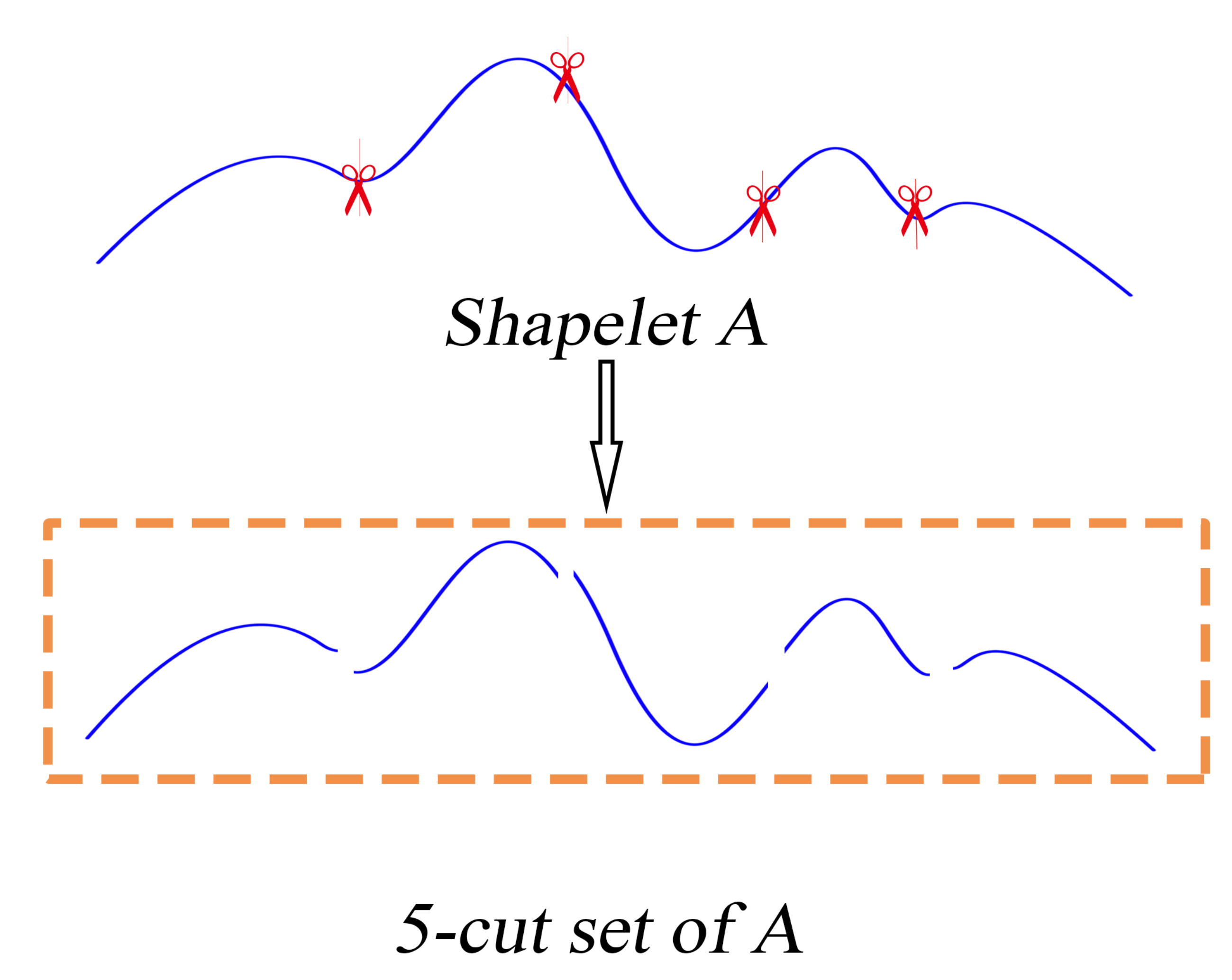}
	\caption{Example of $n$-cut set.}
	\label{Fig 5}
\end{figure}

A simple example of $n$-cut set is shown in Fig. \ref{Fig 5}. Based on the aforementioned conceptions, a key proposition is given as follows:

~\\
\begin{prop}
\emph{$A$ is a vector set coming from a time series set $B$'s shapelet transform through a single shapelet $(\mathbf{s},z)$ by Fixed Distance, and $C$ is another vector set coming from $B$'s shapelet transform through $D$ by Fixed Distance too, here $D$ is an arbitrary $n$-cut set of $(\mathbf{s},z)$ with any $n>1$. Then for any time series in $B$, its corresponding vector in $A$ has the same Euclidean norm with its corresponding vector in $C$.}
\label{prop:1}
\end{prop}

\begin{proof}
For any time series $(\mathbf{x},y)$ in set $B$, it's needed to prove that its shapelet transform through $(\mathbf{s},z)$ by Fixed Distance has the same Euclidean norm with its shapelet transform through any $n$-cut set $D$ of $(\mathbf{s},z)$ by Fixed Distance.

The following is the representation of the time series $\mathbf{x}$, the shapelet $\mathbf{s}$, and the $n$-cut set $D$.
\begin{subequations}
\begin{gather}
\mathbf{x}=(x_1,x_2,\cdots,x_m) \\
\mathbf{s}=(s_i,s_{i+1},\cdots,s_{i+k}) \quad 1\leq i \leq m,\  i+k\leq m \\
D=\{ \mathbf{s}_1,\mathbf{s}_2,\cdots,\mathbf{s}_n \}  \quad  1\leq n \leq k+1\\
\begin{array}{c}
t_j={\rm the \ terminal \ position \ of} \ \mathbf{s}_j    \\ 1\leq j \leq n ,\  i\leq t_1 < t_2 < \cdots < t_n = i+k
\end{array}
\\
\mathbf{s}_1=(s_i,s_{i+1},\cdots,s_{t_1}) \\
\mathbf{s}_h=(s_{t_{h-1}+1},s_{t_{h-1}+2},\cdots,s_{t_h}) \quad 2\leq h \leq n
\end{gather}
\end{subequations}

Next formulas represent the vectors transformed through $\mathbf{s}$ and $D$ respectively by Fixed Distance.
\begin{equation}
V_{\mathbf{x}/\mathbf{s}}=(\sqrt{\sum_{p=i}^{i+k} {(x_p-s_p)^2}})
\end{equation}
\begin{equation}
\begin{split}
V_{\mathbf{x}/D}=(\sqrt{\sum_{p=i}^{t_1} {(x_p-s_p)^2}},\sqrt{\sum_{p=t_1+1}^{t_2} {(x_p-s_p)^2}},\cdots, \\ \sqrt{\sum_{p=t_{n-1}+1}^{t_n} {(x_p-s_p)^2}} )
\end{split}
\end{equation}

The next step is to prove that these two vectors have the same Euclidean norm. Notice that if all the shapelets in $D$ are concatenated in sequence then the shapelet $\mathbf{s}$ is acquired, hence,
\begin{equation}
\begin{split}
||V_{\mathbf{x}/D}||_2 &=\sqrt{ \sum_{p=i}^{t_1} {(x_p-s_p)^2}+\sum_{q=2}^{n}{\sum_{p=t_{q-1}+1}^{t_q}{(x_p-s_p)^2}} } \\
                       &=\sqrt{ \sum_{p=i}^{t_n}{(x_p-s_p)^2}  }\\
  & =\sqrt{ \sum_{p=i}^{i+k}{(x_p-s_p)^2} } \\
 & = ||V_{\mathbf{x}/\mathbf{s}}||_2
\end{split}
\end{equation}

And this is exactly what is needed.
\end{proof}
~\\

According to the shaplet transform defined in Definition \ref{defi 1}, the word `norm' in the Proposition \ref{prop:1} actually means the similarity measurement between the time series and the shapelets representing a local feature set of the class `$z$'. Therefore, in Proposition \ref{prop:1}, for a very discriminative shapelet $(\mathbf{s},z)$, the Euclidean norm of a corresponding vector of the time series whose class label is `$z$' should be as small as possible, and it of the other vectors should be as large as possible (Here, the small value of similarity measurement means a high similarity and the big one means a low similarity). Hence, this proposition actually ensures that the similarity measurement is an invariant under those two different shapelet transform processes. An intuitive and simple example is shown in Fig. \ref{Fig 6}.

\begin{figure}[!t]
\centering
\includegraphics[width=\linewidth]{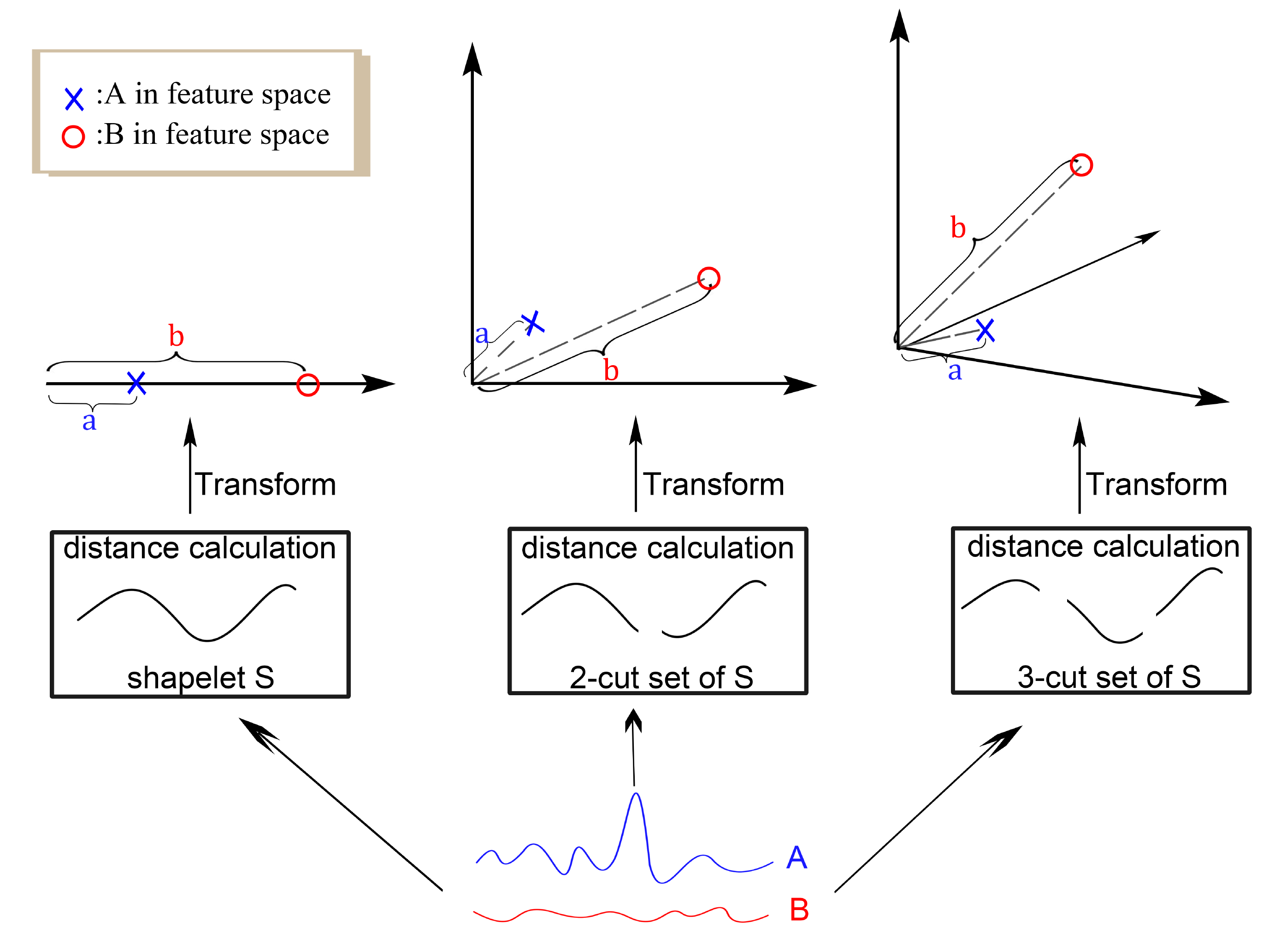}
\caption{An intuitive example to interpret Proposition \ref{prop:1}. A and B are two time series, and they are transformed through a shapelet, its 2-cut set, and its 3-cut set respectively by Fixed Distance. Though A has different representative vectors in different feature spaces, these vectors keep the same length (the Euclidean norm) in different feature spaces. And so does B. Hence one keeps close to origin and the other keeps far from origin in all the three feature spaces.}
\label{Fig 6}
\end{figure}

Furthermore, the following corollary can be proven based on Proposition \ref{prop:1}.

\newtheorem{corl}{Corollary}
~\\
\begin{corl}
\emph{$A$ is a vector set coming from a time series set $B$'s shapelet transform through a shapelet set $C$ by Fixed Distance, $D \subset C$,  $E=\bigcup_{\mathbf{s} \in D} {(n_{\mathbf{s}}}$-cut set of $\mathbf{s})$, and $F=C-D \cup E$. $G$ is another vector set coming from $B$'s shapelet transform through $F$ by Fixed Distance. Then for any time series in $B$, its corresponding vector in $A$ has the same Euclidean norm with its corresponding vector in $F$.}
\label{corl:1}
\end{corl}

\begin{proof}
The proof begins with the following notation. $C$ is a shapelet set represented as follows:
\begin{equation}
C=\{\mathbf{s}_1,\mathbf{s}_2,\cdots,\mathbf{s}_m\}.
\end{equation}
 $D$ is a subset of $C$, it is represented as follows:
\begin{equation}
D=\{ \mathbf{s}_{k_1}, \mathbf{s}_{k_2},\cdots,\mathbf{s}_{k_l} \}.
\end{equation}
$E$, as defined in the corollary, is the union set of all the $n_s$-cut sets of shapelets in $D$, and it is represented as follows:
\begin{equation}
E=\{ \mathbf{s}^{1}_{k_1},\cdots,\mathbf{s}^{n_{k_1}}_{k_1},\mathbf{s}^{1}_{k_2},\cdots,\mathbf{s}^{n_{k_2}}_{k_2},\cdots,\mathbf{s}^{1}_{k_l},\cdots,\mathbf{s}^{n_{k_l}}_{k_l} \}.
\end{equation}
Besides, $C-D$ is represented as follows:
\begin{equation}
C-D=\{ \mathbf{s}_{t_1}, \mathbf{s}_{t_2},\cdots,\mathbf{s}_{t_{m-l}} \}.
\end{equation}
Hence, a time series $\mathbf{x}$'s shapelet transform through $C$ and $C-D\cup E$ respectively by Fixed Distance are as follows:
\begin{equation}
V_{\mathbf{x}/C}=\left(
\begin{array}{l}
||V_{\mathbf{x}/\mathbf{s}_{k_1}}||_2,\cdots,||V_{\mathbf{x}/\mathbf{s}_{k_l}}||_2,\\
||V_{\mathbf{x}/\mathbf{s}_{t_1}}||_2,\cdots,||V_{\mathbf{x}/\mathbf{s}_{t_{m-l}}}||_2
\end{array}
 \right),
\end{equation}

\begin{equation}
V_{\mathbf{x}/C-D\cup E}=\left(
\begin{array}{c}
||V_{\mathbf{x}/\mathbf{s}^{1}_{k_1}}||_2,\cdots,||V_{\mathbf{x}/\mathbf{s}^{n_{k_1}}_{k_1}}||_2,\\
\cdots\\
||V_{\mathbf{x}/\mathbf{s}^{1}_{k_l}}||_2,\cdots,||V_{\mathbf{x}/\mathbf{s}^{n_{k_l}}_{k_l}}||_2,\\
||V_{\mathbf{x}/\mathbf{s}_{t_1}}||_2,\cdots,||V_{\mathbf{x}/\mathbf{s}_{t_{m-l}}}||_2
\end{array}
\right).
\end{equation}
Now calculate Euclidean norm of these two vectors,
\begin{equation}
\begin{split}
||V_{\mathbf{x}/C}||_2 &=\sqrt{ \sum_{p=1}^{l} {||V_{\mathbf{x}/\mathbf{s}_{k_p}}||_2^2}+\sum_{q=1}^{m-l}{||V_{\mathbf{x}/\mathbf{s}_{t_q}||_2^2}} } \\
 &=\sqrt{ \sum_{p=1}^{l}{\sum_{r=1}^{n_{k_p}}{||V_{\mathbf{x}/\mathbf{s}^{r}_{k_p}}||_2^2}}+\sum_{q=1}^{m-l}{||V_{\mathbf{x}/\mathbf{s}_{t_q}||_2^2}}  }\\
 & = ||V_{\mathbf{x}/C-D\cup E}||_2
\end{split}
\end{equation}
The second equal sign is hold by Proposition \ref{prop:1}. And this finishes the proof of Corollary \ref{corl:1}.
\end{proof}
~\\

According to this corollary and the analysis tightly after Proposition \ref{prop:1}, if the approximate $n$-cut set of a shapelet $(\mathbf{s},z)$ is added as part of the shapelet basis used for the shapelet transform, then $(\mathbf{s},z)$ in itself can be dropped in that shapelet basis. That's because the classification in the feature space mainly relies on the similarity measurement between transformed data and selected local features. However, the shapelet transform reflects the similarity measurement on the norm of transformed vectors, and Corollary \ref{corl:1} guarantees the identical norms of transformed vectors in different feature spaces constructed by the two different shapelet basis. Hence, Corollary \ref{corl:1} actually says that the original shapelet basis and the changed shapelet basis can construct the nearly same feature space in terms of classification. Therefore,  in a shapelet basis, the substitution of a shapelet with its $n$-cut set is feasible while the eventual purpose is classifying transformed data.

Every long shapelet consists of some short shapelets so that it must have an $n$-cut set including only short sub-shapelets. As stated above, Corollary \ref{corl:1} guarantees that all the shaplets can be replaced by their cut sets of short sub-shapelets. Therefore, selecting the elementary short shapelets is quite enough. Hence, the length of shapelet candidates can be fixed at a small number and these shapelet candidates are called `short isometric shapelet candidates'.

After restricting the length of shapelet candidates to a small constant $h$, the number of shapelet candidates is cut to $n(m-h+1)$ for a $n$-size and $m$-length time series set while there are $\frac{nm(1+m)}{2}$ shapelet candidates in typical shapelet-dependent algorithms.

\subsection{Substituting the Ensemble Classifier in the Feature Space}
\label{subsection:3-3}
After extracting short isometric shapelet candidates, the next step is to select shapelets from these short isometric shapelet candidates. This subsection begins at the analysis of the selection of shapelets. Since it is the key point of substituting the ensemble classifier in the shapelet-dependent feature space.

As previously mentioned, in the process of selecting shaplets from shapelet candidates, if the prime information gain is used as the discrimination index of shapelet candidates, then the evaluation of a shapelet candidate can not get rid of searching a prime segmentation for data represented in a real number axis (see Subsection \ref{subsection:2-1}). And this operation is time-consuming.

On the other hand, the selected shapelets are also relevant to the distribution of transformed time series in the feature space. If nothing can be ensured of this distribution in the feature space, things become fairly complicated. The following illustration is shown to explain this point.

In the ensemble shapelet transform algorithm, the classifier is often an ensemble one of almost all kinds of typical classifiers trained in the feature space. It costs much time to train such an elaborate ensemble classifier. However, it is necessary to do that on account of the uncertainty of the distribution of data in the feature space. In that algorithm, each base classifier tries to catch one possible distribution of data in the feature space, while using the single classifier must undertake the risk of misestimating the data distribution. One radical cause of that plight is that the discrimination index of shapelets is designed without a thought about the combination of shapelets. That is to say, mere good separability in each dimension ensures little in the total space.

This illustration exemplifies that considering more than discrimination of a single shapelet is required while selecting shapelets from shapelet candidates. The combination of shapelets is another important issue while selecting shapelets. One example is that if the selected shapelet set constructing the feature space can ensure a high linear separability of the data distribution, the classification conducted in the feature space will become comparatively easy. In the Euclidean feature space constructed by the shapelet transform, each shapelet represents a dimension, and values in that dimension represent the distance to the shapelet. If projections of transformed data onto each dimension is completely separable, which means there is a cut line completely dividing data from different two classes, the transformed data can be surely linearly separable in the total space. However, in practice, the majority of axes is the one where the projection onto it can not be completely divided, even though shapelets are selected as discriminative as they can. Hence, the difficulty is  to ensure a high degree of linear separability while projections of data onto a single axis can not be completely divided. Thanks to the next Proposition \ref{prop:2},  this goal can be achieved by using the generalized Rayleigh quotient as the selection priority of shapelet candidates.

Before stating and proving Proposition \ref{prop:2}, the generalized Rayleigh quotient used in this paper should be defined. For two different classes of data distributed on a single real axis, the generalized Rayleigh quotient is calculated by the next formula:

\begin{equation}
GRQ(A,B)=\frac{|\mu(A)-\mu(B)|}{\sigma^2(A)+\sigma^2(B)}
\label{EQ:GRQ}
\end{equation}

where $GRQ$ is a short hand of `generalized Rayleigh quotient', $\mu(\cdot)$ is the function to get the mean value, and $\sigma^2(\cdot)$ is the function to get the variance. For the shapelet candidate in a binary TSC problem, after transforming the original time series to real numbers by distance calculation, Eq. (\ref{EQ:GRQ}) can be used to calculate the $GRQ$ value.

The next part is the statement and the proof of the crucial Proposition.
~\\
\begin{prop}
\label{prop:2}
\emph{For a binary TSC problem in a Euclidean feature space constructed by the shapelet transform, the generalized Rayleigh quotient of the projections of transformed data onto any single axis has a positive correlation with the degree of the linear separability of the transformed data in the total space.}
\end{prop}
\begin{proof}
Firstly, some notations are necessary:
\begin{align*}
& X\backslash Y\backslash X'\backslash Y'  & : & \   a \ real \ random \ vector  \\
& \Omega & : & \ set \ of \ all  \ real  \ random \ vectors  \\
& \mu : \Omega \mapsto \cup_{k=1}^{\infty} R_k& : & \ a \ function \ getting  \ the \ mean  \\
& & & \ vector \ of \ the \ random \ vector \\
& {\mu}_j : \Omega \mapsto R & : & \ a \ function \ equals \ to \ \pi_j\circ \mu \ where \\
& & &  \ \pi_j \ is \ a \ projection \ map \ which \\
& & & \ gets \ the \ value \ of \ a \ vector's \ j\textrm{-}th \\
& & & \ dimension \\
& [ \ ]^2 & : & \ get \ a \ vector \ which  \ the \ value \ of \\
& & & \ each \ dimension  \ equals  \ the \ square \\
& & & \ of \ the \ original \ value \\
& {\sigma}^2 : \Omega \mapsto \cup_{k=1}^{\infty} R_k & : & \ a \ function \ getting \ the \ diagonal \\
& & & \ vector \ of \ the \ covariance \ matrix \\
& & & \ of \ the \ random \ vector \\
& {\sigma}_j^2 : \Omega \mapsto R & : & \ a \ function \ equals \ to \ \pi_j\circ {\sigma}^2 \ where \\
& & &  \ \pi_j \ is \ a \ projection \ map \ which \\
& & & \ gets \ the \ value \ of \ a \ vector's \ j\textrm{-}th \\
& & & \ dimension \\
& {\Sigma}_{X\backslash Y\backslash X'\backslash Y'} & : & \  the \ covariance \ matrix \ of \ the  \\
& & & \ random \ vector  \\
& LS & : & \ the \ degree \ of  \ the \ linear  \ separability \\
& & &  \ of \ vectors \ of \ two \ different \ classes
\end{align*}

Based on the above notations, the proof can start. Assume that there are two time series classes whose corresponding sets are $A$ and $B$, a shapelet set $C$ whose cardinality is $k$. $A$'s shapelet transform through $C$ by some metric (see definition \ref{defi 1}) forms the sample set of a $k$-dimensional real random vector $X$. The same process of $B$ forms the sample set of a $k$-dimensional real random vector $Y$. By the central limit theorem,  these two random vectors can be presumed to be independent spherical normally distributed. That is:
\begin{align}
X &\sim N(\mu(X),\Sigma_X)   \\
Y &\sim N(\mu(Y),\Sigma_Y)
\end{align}
And we naturally define the degree of the linear separability by the following formula:
\begin{equation}
\begin{split}
LS= \max_{\boldsymbol{w}} \ &  P((X-Y)\cdot \boldsymbol{w}>0)   \\
  &  ||\boldsymbol{w}||=1, \boldsymbol{w} \in R_k
\end{split}
\label{f:25}
\end{equation}

For proving Proposition \ref{prop:2}, we need to prove that $|\mu_{j}(X)-\mu_{j}(Y)|$ has a positive correlation with $LS$ and $\sigma_j^2{(X)}+\sigma_j^2{(Y)}$ has a negative correlation with $LS$ for any $j$ in scope.

For a fixed $\boldsymbol{w}$, the following result can be deduced by the previous conditions.
\begin{equation}
(X-Y) \cdot \boldsymbol{w} \sim N((\mu(X)-\mu(Y)) \cdot \boldsymbol{w},  (\sigma^2{(X)}+\sigma^2{(Y)}) \cdot [\boldsymbol{w}]^2)
\label{f:26}
\end{equation}

For the prime normal vector $\boldsymbol{w}$ of Eq. (\ref{f:25}), it has that:
\begin{gather}
(\mu(X)-\mu(Y)) \cdot \boldsymbol{w} \geq 0    \label{f:27} \\
(\mu_j (X)- \mu_j (Y)) \cdot \boldsymbol{w}_j \geq 0   \quad 1\leq j \leq k
\label{f:28}
\end{gather}

For the first inequality, that is because $\boldsymbol{w}$ can be substituted with $-\boldsymbol{w}$ if the inequality does not hold. Then $-\boldsymbol{w}$ is better than $\boldsymbol{w}$, which contradicts that $\boldsymbol{w}$ is the prime normal vector. It is the same with the second inequality, if it is less than $0$, just substitute $\boldsymbol{w}_j$ with $-\boldsymbol{w}_j$ to get a better normal vector $\boldsymbol{w}'$. And it contradicts that $\boldsymbol{w}$ is the prime normal vector. Fig.  \ref{Fig 12} shows these cases.

Assuming that the best $LS$ and the prime normal vector $\boldsymbol{w}$ have been acquired,  if substitute $X$ and $Y$ with $X'$ and $Y'$ which are same with $X$ and $Y$ except for a specific $j$ where $|\mu_j(X')-\mu_j(Y')|>|\mu_j(X)-\mu_j(Y)|$, then the following results hold from Eq. (\ref{f:26}), (\ref{f:27}) and (\ref{f:28}).
\begin{gather}
\mu((X'-Y')\cdot \boldsymbol{w'})>\mu((X-Y)\cdot \boldsymbol{w}) \geq 0 \\
\sigma^2((X'-Y')\cdot \boldsymbol{w'})=\sigma^2((X-Y)\cdot \boldsymbol{w})
\end{gather}
Here $\boldsymbol{w'}$ equals $\boldsymbol{w}$ on all dimensions except for the $j$-th value. And $\boldsymbol{w'}_j=sign((\mu_j(X')-\mu_j(Y')) \cdot |\boldsymbol{w}_j|$, which aims to make $(\mu_j(X')-\mu_j(Y')) \cdot \boldsymbol{w'}_j$ positive. Combining with (\ref{f:25}), they shows that:
\begin{equation}
LS'\geq P((X'-Y') \cdot \boldsymbol{w'}>0)>P((X-Y) \cdot \boldsymbol{w}>0)=LS
\label{f:31}
\end{equation}
\begin{figure}[!t]
\centering
\includegraphics[width=0.5\linewidth]{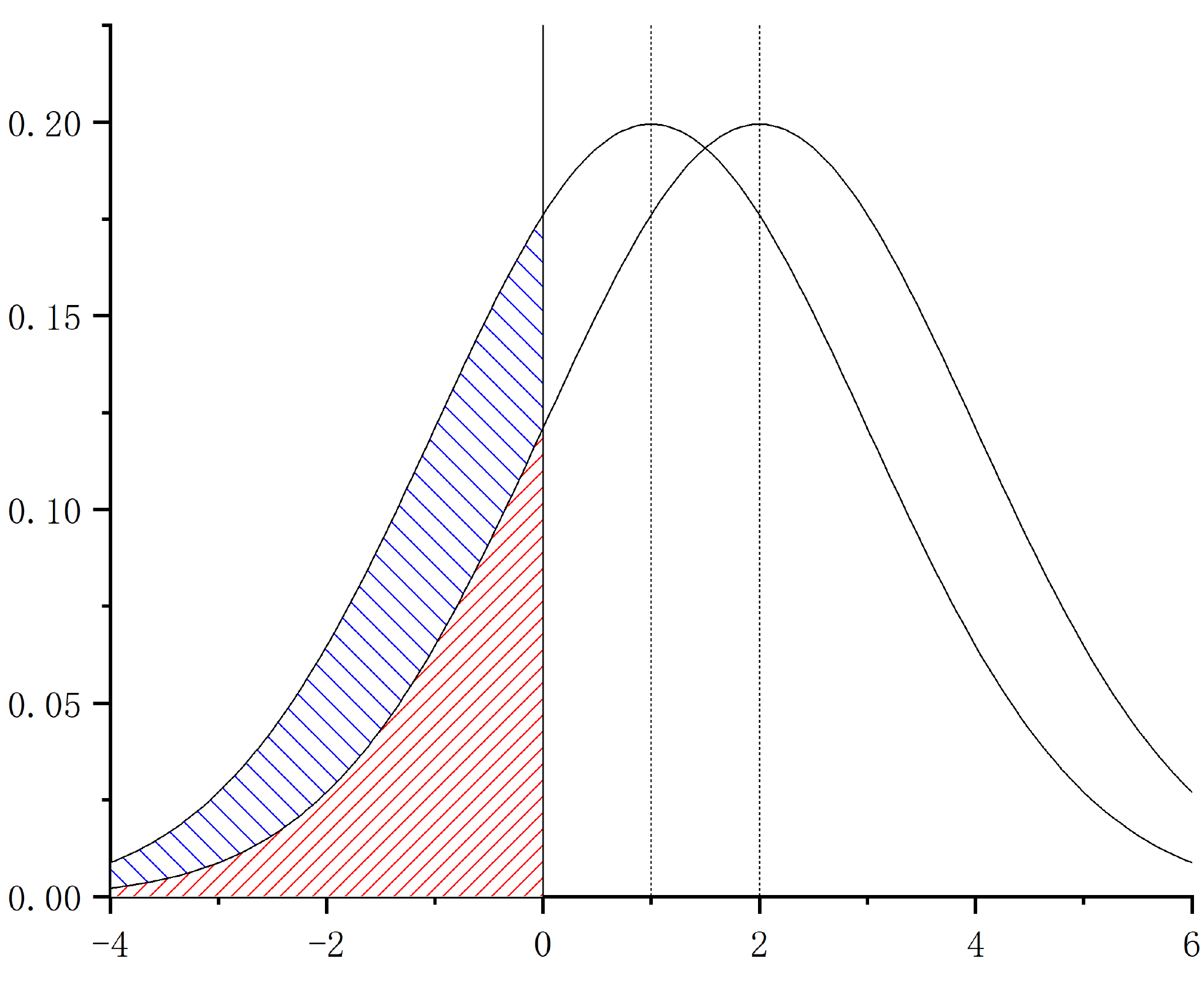}
\caption{Two independent normally distributed variables with the identical variance. The one whose mean is larger has a higher probability of being positive than the other. That's because the area under two curves are both $1$ and the curve with the larger mean possesses a smaller area left to $x=0$. Hence the curve with the larger mean has a larger area right to $x=0$, which means a higher probability of getting a positive value. }
\label{Fig 12}
\end{figure}
Fig. \ref{Fig 12} shows that situation. And this means that $|\mu_{j}(X)-\mu_{j}(Y)|$ has a positive correlation with $LS$ for arbitrary $j$ in scope.

Similarly, keep the assumption that the best $LS$ and the prime normal vector $\boldsymbol{w}$ of the optimization problem \ref{f:25} have been acquired, and substitute $X$ and $Y$ with $X'$ and $Y'$ which are same with $X$ and $Y$ except for a specific $j$ where $\sigma_j^2{(X')}+\sigma_j^2{(Y')}<\sigma_j^2(X)+\sigma_j^2(Y)$, then  the following results hold from Eq. (\ref{f:26}):
\begin{gather}
\mu((X'-Y')\cdot \boldsymbol{w})=\mu((X'-Y')\cdot \boldsymbol{w}) \geq 0 \\
\sigma^2((X'-Y')\cdot \boldsymbol{w})<\sigma^2((X-Y)\cdot \boldsymbol{w})
\end{gather}
\begin{figure}[!t]
\centering
\subfigure{\includegraphics[width=0.49\linewidth]{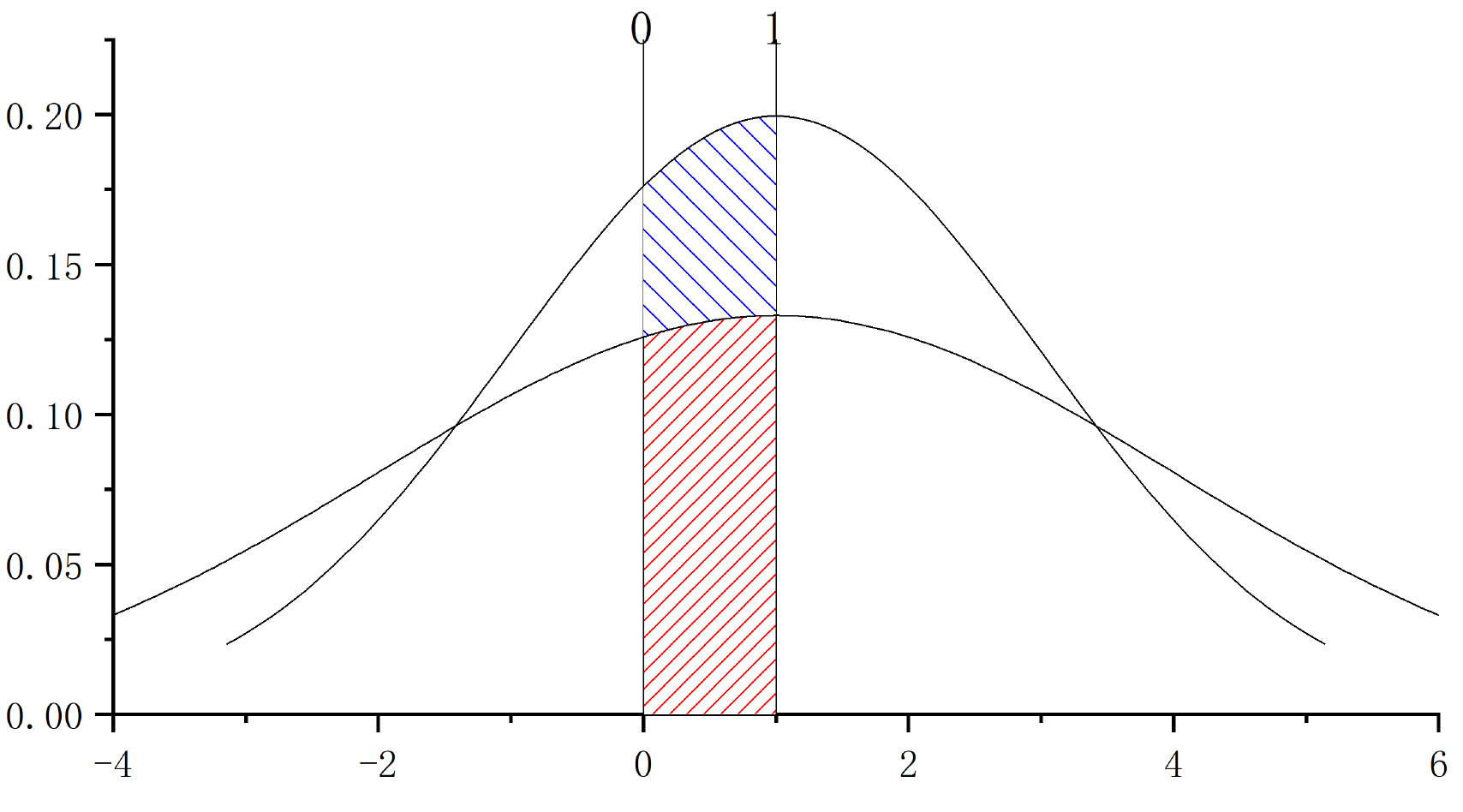}}
\subfigure{\includegraphics[width=0.49\linewidth]{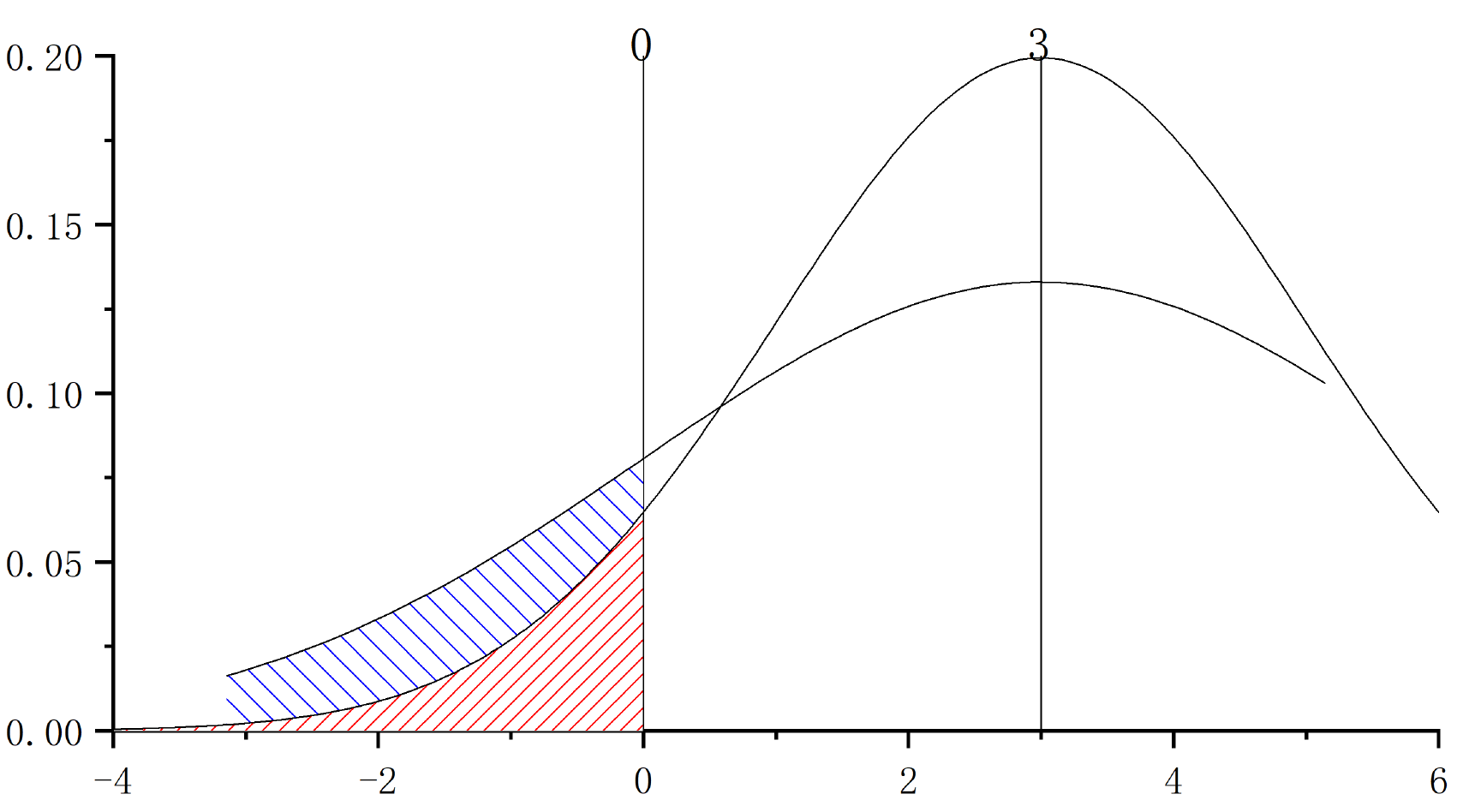}}
\caption{For two independent normally distributed variables with the identical positive mean, the one whose variance is smaller has a higher probability of being positive than the other. There are two situations. The first one is that the left intersection point of two curves is left to $x=0$, which is shown in the first picture. The second one is that the left intersection point of two curves is right to $x=0$, which is shown in the second picture.}
\label{Fig 13}
\end{figure}
Combining with (\ref{f:25}), Eq. (\ref{f:31}) holds too. Fig. \ref{Fig 13} shows that situation. And this means that $\sigma_j^2{(X)}+\sigma_j^2{(Y)}$ has a negative correlation with $LS$ for arbitrary $j$ in scope. Then the proof of Proposition \ref{prop:2} is completed.
\end{proof}
~\\

 From the calculation formula of the $GRQ$ (Eq. (\ref{EQ:GRQ})), it can be seen that the $GRQ$ surely has the ability to quantify the separation extent of data in a real number axis. Hence the large $GRQ$ value guarantees the discrimination of a single selected shapelet. The prime information gain also has the ability. But the large $GRQ$ value in each single axis can ensure a high degree of the linear separability in the total space, whereas the large prime information gain can not do that. And this advantage is exactly what Proposition \ref{prop:2} states. This property of the $GRQ$ is the most important advantage of it compared with the prime information gain.

In Proposition \ref{prop:2}, each axis of the feature space actually corresponds to the real number axis of a shapelet. And the projections of transformed time series onto each axis are actually the distances from time series to the corresponding shapelet. Hence, Proposition \ref{prop:2} can ensure a high degree of the linear separability of data in the feature space under the precondition of selecting shapelets with the large $GRQ$ value. As a result, if shapelets are selected according to the rank of $GRQ$ values, training a single SVM to classify data also works at most cases. Once the ensemble classifier is substituted with a single SVM, time consumption is largely reduced for the reduction of training costs. Moreover, using the $GRQ$ as the selection priority of shapelet candidates also successfully gets rid of the time-consuming operation that searching the best segmentation for different data in a real number axis.

\subsection{Outline and Time Complexity of Training  SIST}
\label{subsection:3-4}
Considering the small delay which may happen in time series, a small relaxation is added to Eq. (\ref{EQ 2}):

\begin{equation}
\label{EQ 3}
\begin{split}
& Relaxed\_Fixed\_Dist(\mathbf{s},\mathbf{x})=\\
& min \left\lbrace d(\mathbf{s},P) \Bigg |
\begin{array}{c}
P=(x_{i_1},\cdots,x_{i_k}), \\
 i_1<i_2<\cdots <i_k<j+k+r,\\
 -l \leq i_1-j \leq r, \\
1 \leq i_1 \leq Len(\mathbf{x})-k+1
\end{array}
\right\rbrace
\end{split}
\end{equation}
where both $l$ and $r$ are very small relaxation constants, $P$ is the ordered subsequence which can be continuous or discontinuous, $Len(\cdot)$ is the function to get the length of time series, and other parameters are same with them in Eq. (\ref{EQ 2}).

The algorithmic outline of training  SIST  is given in Algorithm \ref{alg1}.

\begin{algorithm}
\caption{Train the SIST classifier}
\label{alg1}
\begin{algorithmic}[1]
\REQUIRE $D$ the binary data set, $L$ the shapelet length, $r_a$ the left relaxation factor, $r_b$ the right relaxation factor, $N$ the number of shapelets
\ENSURE Shapelet set $S$ whose cardinality is $N$, a SVM classifier $C$ with the linear kernel
\STATE initialize a null set $T$;
\FOR{each time series $t \in D $}
\STATE initialize a $L$-length shapelet candidate set $T_t$ by a $L$-length window sliding from the initial point to the end point of $t$ with step length $1$;
\STATE $T=T \cup T_t$;
\ENDFOR
\STATE initialize a priority queue $S$;
\FOR{each shapelet candidate $s \in T$}
\FOR{each time series $t \in D $}
\STATE calculate Relaxed Fixed Distance between $s$ and $t$ by Eq. (\ref{EQ 3});
\ENDFOR
\STATE calculate the generalized Rayleigh quotient of the calculated distances by Eq. (\ref{EQ:GRQ});
\STATE set the generalized Rayleigh quotient as the priority of $s$;
\STATE $S=S \cup s$;
\ENDFOR
\STATE $S$=subqueue of $S$ from position $0$ to $N-1$
\STATE do $D$'s shapelet transform through $S$ by Relaxed Fixed Distance (the process is stated in definition \ref{defi 1}) and get a vector set $\Omega$;
\STATE train a SVM $C$ with linear kernel in $\Omega$;
\end{algorithmic}
\end{algorithm}

For a data set $D$ of $n$ time series of length $m$, fixing the length of shapelet candidates at $k$, the time consumption of training  SIST  can be estimated as follows:
\begin{equation}
\begin{split}
T(SIST) &=  \mathcal{O}(n(m-k+1)T(SP)+T(R)+T(SVM))  \\
        &\approx  \mathcal{O}(n^2mk+nm\log(nm)+Nn^2)   \\
        &=  \mathcal{O}(n^2\max(mk,N))   \\
        &<  \mathcal{O}(n^2m^2)   \\
\end{split}
\end{equation}
where $T(SP)$, which is multiplied by the number of shapelet candidates, means the time complexity of calculating the selection priority of shapelet candidates, $T(R)$ means the time complexity of ranking all the shapelet candidates, and $T(SVM)$ means the time complexity of training a SVM with the linear kernel in the feature space. Since the components of transformed vectors can be acquired in the process of calculating the selection priority, there is no extra time needed for the shapelet transform based on the selected shapelet basis. Empirically, $N$, which is the dimension of the feature space, is set less than $m^2$. And $k$ is obviously less than $m$. Therefore, the inequality is satisfied. And in the ensemble ST algorithm, the time complexity is $\mathcal{O}(n^2m^4)$ \cite{Hills2013ClassificationOT,Bostrom2015BinaryST}. This time complexity analysis shows some  advantages of the proposed strategies. And in the next section, some experimental evidences of the advantages are given.

\begin{table}[!t]
	\renewcommand{\arraystretch}{1.8}
	\caption{Hyperparameter Set}
	\label{Table 0}
	\centering
	\begin{tabular}{|c|c|}
		\hline
		delete overlap shapelet & {true, false}\\
		\hline
		shapelet length & {3,4}\\
		\hline
		left and right relaxation factor & {3,4}\\
		\hline
		shapelet number & \tabincell{c}{10,50,100,250,500,750, \\ 1000,1250,1500,2000}\\
		\hline
	\end{tabular}
\end{table}

\begin{table*}[h]
\renewcommand{\arraystretch}{1.8}
\caption{Data Attributes and Hyperparameters Chosen by 10-fold Cross Validation.
	DO:Delete Overlap.
	SL: Shapelet Length.
	Relaxation Factor: RF.
	Shapelet Number: SN.
	PPOC: ProximalPhalanxOutlineCorrect.
}
\label{Table 1}
\centering
\begin{tabular}{|c|c|c|c|c|c|c|}
\hline
\multicolumn{3}{|c|}{Data Attribute} & \multicolumn{4}{|c|}{Hyperparameter of SIST}\\
\hline
Dataset & Train$\backslash$Test & Length & DO & SL  & L$\backslash$R RF & SN \\
\hline
Coffee & 28$\backslash$28 & 286 & true & 3 & 3$\backslash$3 & 10\\
\hline
DistalPhalanxOutlineCorrect & 600$\backslash$276 & 80 & true & 4 & 4$\backslash$4 & 2000\\
\hline
Earthquakes & 322$\backslash$139 & 512 & true & 3 & 3$\backslash$3 & 10\\
\hline
ECG200 & 100$\backslash$100 & 96 & true & 3 & 4$\backslash$3 & 1500\\
\hline
ECGFiveDays & 23$\backslash$861 & 136 & false & 3 & 3$\backslash$3 & 2000\\
\hline
GunPoint & 50$\backslash$150 & 150 & false & 4 & 3$\backslash$3 & 1250\\
\hline
Ham  & 109$\backslash$105 & 431 & true & 3 & 4$\backslash$3 & 50\\
\hline
Herring & 64$\backslash$64 & 512 & true & 4 & 4$\backslash$3 & 10\\
\hline
ItalyPowerDemand & 67$\backslash$1029 & 24 & false & 3 & 4$\backslash$4 & 100\\
\hline
MiddlePhalanxOutlineCorrect & 600$\backslash$291 & 80 & true & 3 & 4$\backslash$4 & 1500\\
\hline
MoteStrain & 20$\backslash$1252 & 84 & true & 4 & 3$\backslash$4 & 250\\
\hline
PPOC & 600$\backslash$291 & 80 & true & 4 & 3$\backslash$3 & 2000\\
\hline
SonyAIBORobotSurface1 & 20$\backslash$601 & 70 & false & 4 & 4$\backslash$3 & 50\\
\hline
SonyAIBORobotSurface2 & 27$\backslash$953 & 65 & false & 3 & 3$\backslash$3 & 2000\\
\hline
Strawberry & 613$\backslash$370 & 235 & false & 4 & 4$\backslash$4 & 1250\\
\hline
TwoLeadECG & 23$\backslash$1139 & 82 & true & 3 & 4$\backslash$3 & 50\\
\hline
Wine & 57$\backslash$54 & 234 & true & 3 & 3$\backslash$3 & 250\\
\hline
\end{tabular}
\end{table*}

\section{Experimental Evidences}
\label{section:experiment}
\subsection{Compared Algorithms and Datasets}
\label{subsection:4-1}
To show the advantages of the proposed strategies,  SIST  is compared with the state-of-the-art. According to the research of A. Bagnall et al.\cite{Bagnall2016TheGT}, the following eight algorithms are  the state-of-the-art algorithms. They are COTE (Collection of Transformation Ensembles) \cite{Bagnall2015TimeSeriesCW}, ST (ensemble edition) \cite{Bostrom2015BinaryST}, BOSS (Bag of SFA Symbols, ensemble edition) \cite{Schfer2014TheBI}, EE (Elastic Ensemble) \cite{Lines2014TimeSC}, DTWF (Dynamic Time Warping Features) \cite{Kate2015UsingDT}, TSF (Time Series Forest) \cite{Deng2013ATS}, TSBF (Time Series Bag of Features) \cite{Baydogan2013ABF}, and LPS (Learned Pattern Similarity) \cite{Baydogan2015TimeSR}.

The proposed algorithm has some hyperparameters to be selected, and Table \ref{Table 0} shows the hyperparameter sets used in the experiments. Line 1 is about whether removing the overlap shapelets. The item `true' means that if two shapelets overlap, then the one with the lower selection priority is removed. And this choice may make the final shapelet number  smaller than the target number. This item means each discriminative position in a time series is at most once selected for classification. It is to say,  users want a balance among every discriminative position's contribution to classification in case that one position's information covers others. The choice `false' is just the opposite, which means users want to see that the more discriminative the position is, the stronger effects it has on the classification. Line 2 is about the length of shapelet candidates. Hyperparameters in line 3 decide the relaxation extent of the Fixed Distance calculation. Line 4 is about the number of shapelets used for the shapelet transform.

As for data sets, only the binary data sets from UCI data sets are selected, since Eq. (\ref{EQ:GRQ}) is restricted on binary TSC problems. In principle, for a multiclass TSC problem, the `one vs all' strategy can solve it. However, the time complexity will be multiplied by the number of classes under this strategy. Hence, the `one vs all' strategy may be not a good method for using the proposed strategies in a multiclass TSC problem. Multiclass TSC problems have a different structure from binary TSC problems, and this is another issue needing to be explored rather than tackled by a simple `one vs all' strategy. Actually, the future work is regarding to expanding the proposed strategies to multiclass TSC problems. But this paper mainly focuses on the theoretical evidences and some advantages of the proposed strategies. Table \ref{Table 1} shows details of data sets and hyperparameter selections of  SIST.
Since these data sets are commonly used benchmark data sets in the research area of TSC, the default training/testing split settings are used here for a fair comparison.

\subsection{Comparison Experiments}
\label{subsection:4-2}
This subsection will show advantages of the proposed strategies by comparing  SIST  with the state-of-the-art algorithms in terms of accuracy and efficiency. The purpose of the proposed strategies is  to reduce the time complexity with little loss of the accuracy achieved by the state-of-the-art on binary TSC problems. Hence,  it should be checked that  SIST  indeed achieves a top classification accuracy on most of binary TSC problems.

\begin{table*}[!t]
\renewcommand{\arraystretch}{1.8}
\caption{Accuracy Comparison on Binary Datasets with eight Top TSC Algorithms. 
	DPOC: DistalPhalanxOutlineCorrect. 
	ECGFD: ECGFiveDays.
	IPD: ItalyPowerDemand.
	MPOC: MiddlePhalanxOutlineCorrect.
	PPOC: ProximalPhalanxOutlineCorrect
	SAIBORS1: SonyAIBORobotSurface1.
	SAIBORS2: SonyAIBORobotSurface2.
	TLECG: TwoLeadECG.
}
\label{Table 2}
\centering
\begin{tabular}{|c|cccccccc|c|}
\hline
Dataset & DTWF & ST & BOSS & TSF & TSBF & LPS & EE & COTE & SIST\\
\hline
Coffee & 0.973 & 0.995 & 0.989 & 0.989 & 0.982 & 0.95 & 0.989 & $\bm{1.000}$ & $\bm{1.000}$\\
\hline
DPOC & 0.796 & 0.829 & 0.815 & 0.810 & 0.816 & 0.767 & 0.768 & 0.804 & $\bm{0.830}$\\ % DistalPhalanxOutlineCorrect
\hline
Earthquakes & 0.748 & 0.737 & 0.746 & 0.747 & 0.757 & 0.668 & 0.735 & 0.747 & $\bm{0.871}$\\
\hline
ECG200 & 0.819 & 0.840 & $\bm{0.891}$ & 0.868 & 0.847 & 0.808 & 0.881 & 0.873 & 0.86\\
\hline
ECGFD & 0.907 & 0.955 & 0.983 & 0.922 & 0.849 & 0.840 & 0.847 & $\bm{0.986}$ & 0.978\\
\hline
GunPoint & 0.964 & $\bm{0.999}$ & 0.994 & 0.961 & 0.965 & 0.972 & 0.974 & 0.992 & 0.967\\
\hline
Ham & 0.795 & 0.808 & 0.836 & 0.795 & 0.711 & 0.685 & 0.763 & 0.805 & $\bm{0.838}$\\
\hline
Herring & 0.609 & 0.653 & 0.605 & 0.606 & 0.591 & 0.549 & 0.566 & 0.632 & $\bm{0.875}$\\
\hline
IPD & 0.948 & 0.953 & 0.866 & 0.958 & 0.926 & 0.914 & 0.914 & 0.970 & $\bm{0.978}$\\ %ItalyPowerDemand
\hline
MPOC & 0.798 & 0.815 & 0.808 & 0.794 & 0.800 & 0.770 & 0.782 & 0.801 & $\bm{0.897}$\\% MiddlePhalanxOutlineCorrect
\hline
MoteStrain & 0.891 & 0.882 & 0.846 & 0.874 & 0.886 & $\bm{0.917}$ & 0.875 & 0.902 & 0.889\\
\hline
PPOC & 0.829 & 0.881 & 0.867 & 0. 847 & 0.861 & 0.851 & 0.839 & 0.871 & $\bm{0.900}$\\% ProximalPhalanxOutlineCorrect
\hline
SAIBORS1 & 0.884 & 0.888 & 0.897 & 0.845 & 0.839 & 0.842 & 0.794 & $\bm{0.899}$ & 0.839\\ %SonyAIBORobotSurface1
\hline
SAIBORS2 & 0.859 & 0.924 & 0.889 & 0.856 & 0.825 & 0.851 & 0.870 & $\bm{0.960}$ & 0.877\\ %SonyAIBORobotSurface2
\hline
Strawberry & $\bm{0.970}$ & 0.968 & $\bm{0.970}$ & 0.963 & 0.968 & 0.963 & 0.959 & 0.963 & 0.959\\
\hline
TLECG & 0.958 & 0.984 & $\bm{0.985}$ & 0.842 & 0.910 & 0.928 & 0.959 & 0.983 & 0.982\\
\hline
Wine & 0.892 & 0.926 & 0.912 & 0.881 & 0.879 & 0.884 & 0.887 & 0.904 & $\bm{1.000}$\\
\hline
\end{tabular}
\end{table*}

%\begin{table}[!t]
%\caption{Hyperparameter Set}
%\label{Table 1}
%\centering
%\begin{tabular}{|c||c|}
%\hline
%delete overlap shapelet & {true,false}\\
%\hline
%shapelet length & {3,4}\\
%\hline
%left and right relaxation factor & {3,4}\\
%\hline
%shapelet number & {10,50,100,250,750,1250,1500,2000}\\
%\hline
%\end{tabular}
%\end{table}

\begin{figure}[!t]
\centering
\includegraphics[width=\linewidth]{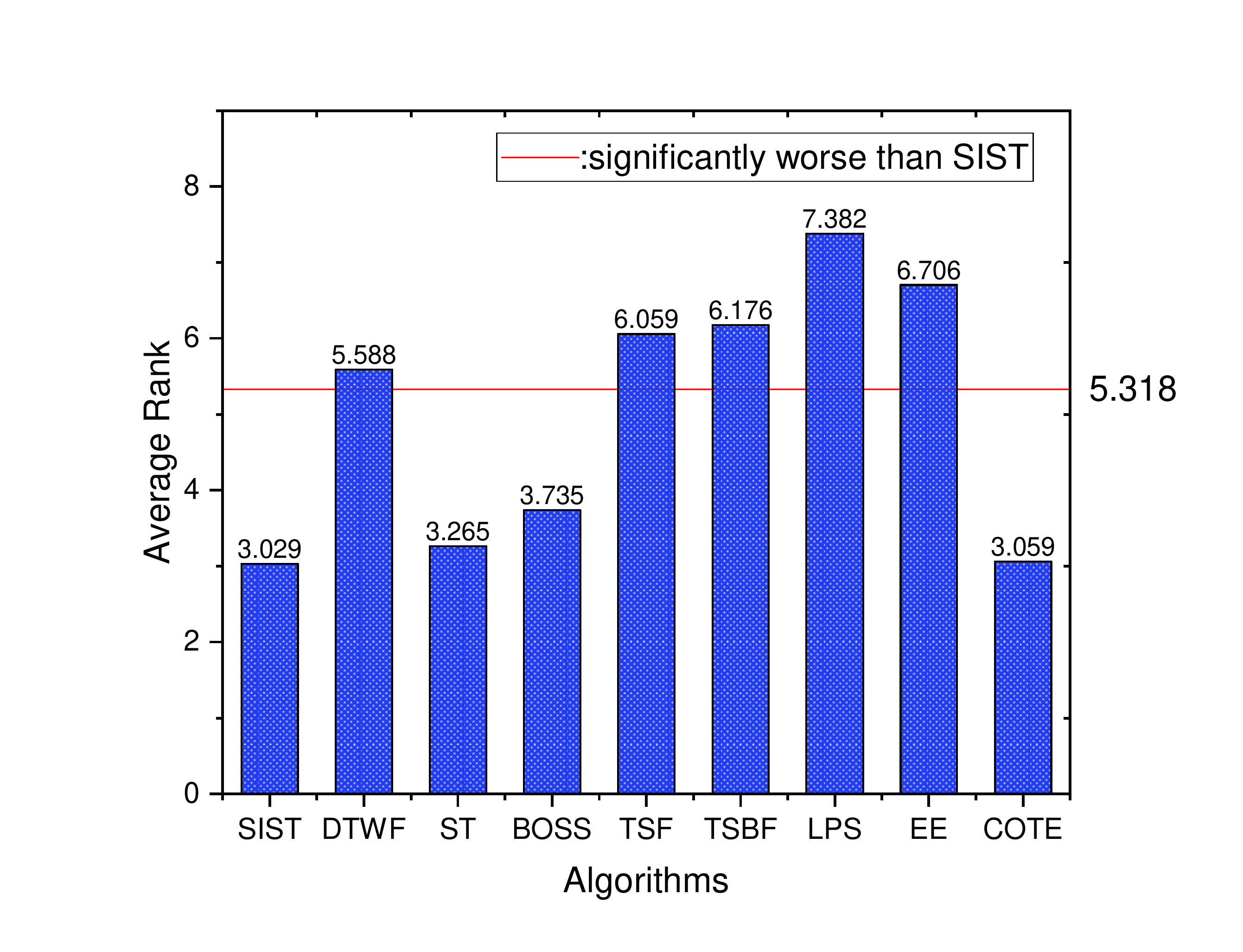}
\caption{Average accuracy rank and Friedman test with 95\% confidence. The average rank in this figure is the arithmetic mean of the accuracy ranks on all the experimental data sets. The red cut line in this figure is the threshold calculated by the Friedman test.}
\label{Fig 9}
\end{figure}

Table \ref{Table 2} presents the accuracy comparison between  SIST  and the state-of-the-art algorithms on many data sets. The accuracy of the compared algorithms comes from the TSC community. 
The hyperparameters of these compared algorithms were  strictly selected to achieve as high accuracy as possible. According to the table,  SIST  performs best on more than a half of the data sets. 
%Actually, the accuracy of  SIST  is fairly close to the best accuracy on the remainder data sets except for `SonyAIBORobotSurface'. 
On the other half of the experimental data sets,  SIST  may not be the best one, 
but it still possesses a reasonable accuracy compared with the state-of-the-art algorithms. 
As it is previously mentioned, the model proposed in the paper tries to reduce the time complexity with the cost of as small accuracy drop as possible. 
The strategies exploited in  SIST  mainly serve the purpose of nearly losslessly reducing the time complexity, 
and from the theoretical analysis in the paper, 
it is also clear that the strategies here will not manifestly boost the classification performance on the basis of the state-of-the-art algorithms. 
Hence, it is sensible here that  SIST does not prevail in all the datasets. 
The key point here is to check whether  SIST also performs well in the datasets where it does not achieve the best accuracy. 
Actually, from the result in the table, the accuracy of  SIST  is fairly close to the state-of-the-art accuracy on these data sets. 
The experimental result here along with the theoretical analysis in Section \ref{section:theory}  
shows that  SIST  does achieve the purpose of keeping accurate enough.

More details about the accuracy comparison can be seen in Fig. \ref{Fig 9}. In this figure, all the algorithms' average ranks are depicted. By putting these average ranks in the same histogram, a holistic view of the comparison is illustrated between  SIST  and the compared algorithms. According to the data in this figure, it can be judged whether  SIST  is significantly better than the compared algorithms by the Friedman test with a 95\% confidence. The red cut line in the figure is a threshold calculated according to the Friedman test. 
SIST  are significantly better than an compared algorithm if its average rank is larger than the threshold.
In Fig. \ref{Fig 9}, the red cut line goes across the bars of the DTWF, the TSF, the TSBF, the LPS, and the EE, which means average ranks of those five algorithms are larger than the threshold. Hence,  SIST  are significantly better than those five algorithms  in terms of classification accuracy. The  four remainder algorithms are the ST, the BOSS, the COTE, and  SIST. And they are in the same level in terms of the classification accuracy. Although  SIST  is not significantly better than the ST, the BOSS, and the COTE in the Friedman test with a 95\% confidence, the average rank of  SIST  is still smaller than the average ranks of those three algorithms. Hence,  SIST  is at least not worse than the ST, the BOSS, and the COTE in terms of the classification accuracy.

\begin{table*}[!t]
	\renewcommand{\arraystretch}{1.8}
	\caption{Time Comparison on Binary Datasets among the first-tier Algorithms with the Same Accuracy Level (CPU : Intel(R) Core(TM) i7-6500U CPU @ 2.50GHz ; RAM : 8GB ; Operating System : Windows 10 x64)}
	\label{Table 3}
	\centering
	\begin{tabular}{|c||ccc||c|}
		\hline
		DataSet & ST & BOSS & COTE & SIST\\
		\hline
		Coffee & 6.3s & 6.6s & 9.5E+03s & $\bm{0.8s}$\\
		\hline
		DistalPhalanxOutlineCorrect & 9.5E+03s & 91.8s & 2.9E+05s & $\bm{88.9s}$\\
		\hline
		Earthquakes & 1.5E+03s & 1.6E+03s & 9.1E+05s & $\bm{61.3s}$\\
		\hline
		ECG200 & 64.5s & 5.0s & 1.3E+04s & $\bm{1.5s}$\\
		\hline
		ECGFiveDays & 7.6s & 1.1s & 1.1E+03s & $\bm{0.3s}$\\
		\hline
		GunPoint & 16.2s & 3.4s & 7.0E+03s & $\bm{0.9s}$\\
		\hline
		Ham & 51.1s & 133.7s & 2.8E+05s & $\bm{6.4s}$\\
		\hline
		Herring & 15.3s & 66.2s & 1.8E+05s & $\bm{5.4s}$\\
		\hline
		ItalyPowerDemand & 16.7s & 0.6s & 88.3s & $\bm{0.1s}$\\
		\hline
		MiddlePhalanxOutlineCorrect & 1.3E+05s & 101.5s & 2.5E+05s & $\bm{34.1s}$\\
		\hline
		MoteStrain & 4.7s & 0.5s & 151.8s & $\bm{0.1s}$\\
		\hline
		ProximalPhalanxOutlineCorrect & 9.8E+03s & 78.3s & 2.6E+05s & $\bm{50s}$\\
		\hline
		SonyAIBORobotSurface1 & 2.0s & 0.5s & 114.2 & $\bm{0.1s}$\\
		\hline
		SonyAIBORobotSurface2 & 4.6s & 1.0s & 212.8s & $\bm{0.2s}$\\
		\hline
		Strawberry & 1.0E+04s & 790.5s & 5.3E+05s & $\bm{279.7s}$\\
		\hline
		TwoLeadECG & 4.4s & 0.6s & 283.0s & $\bm{0.1s}$\\
		\hline
		Wine & 11.2s & 9.0s & 4.3E+04s & $\bm{0.6s}$\\
		\hline
	\end{tabular}
\end{table*}

Based on the above experimental evidences and analysis, the ST, the BOSS, the COTE, and  SIST  can be presently regarded as the first tier candidates of the TSC algorithms in terms of the classification accuracy. Next, the efficiency comparison among the first-tier algorithms will be focused on. Since the purpose of the proposed strategies is to reduce the time complexity with little loss of the accuracy achieved by the state-of-the-art, the precondition of keeping the good accuracy is important. Therefore, these algorithms that may possess the lower time consumption are not considered since they can not keep a high accuracy. In other words, the efficiency comparison focuses only on the first-tier candidates.

Table \ref{Table 3} presents the time consumption comparison on all the experimental data sets in the same environment. Since these four algorithms are all eager learning algorithms, the time consumption is mainly spent in training the classifier. Hence, the time shown in Table \ref{Table 3} is the training time. From  Table \ref{Table 3}, it can be seen that  SIST  has the lowest time consumption on all the experimental data sets.  SIST  has evident advantage compared with the ST, the BOSS, and the COTE in terms of time consumption, which indicates the running time advantage. The COTE has a fairly large training time although it is hot on the heels of  SIST  in Fig. \ref{Fig 9}. To some extent, the huge training time of the COTE makes some troubles in the practical application. On some data sets, the training time of the COTE is several days, whereas  SIST  only spend some minutes in getting a similar classification result. On some data sets, the training time of the compared first-tier algorithms may be some hours or some minutes, whereas  SIST  only spends some seconds. 
%In Table \ref{Table 3},  SIST   outperforms the four first-tier algorithms in terms of the run time, which implies the effectiveness of the proposed strategies.
On some small-scale datasets such as `ECG200’, ‘ECGFiveDays’, and `ItalyPowerDemand’, 
the BOSS algorithm performs competitively compared with  SIST. 
However, as the scale of datasets increases, the time consumption of the BOSS algorithm swells much faster than  SIST. 
For instance, on some large-scale datasets, such as `Ham’, `Herring’, and `Earthquakes’, 
 SIST  performs more efficiently than the BOSS algorithm. 
The rationale behind it is that the BOSS algorithm exploits two technologies, called Discrete Fourier Transform (DFT) and Multiple Coefficient Binning (MCB) \cite{Schfer2014TheBI}. 
The running time of these two technologies is sensitive to the scale of the training data. 
Therefore, the efficiency advantage of  SIST  compared with the BOSS algorithm becomes increasingly evident when the scale of the training data grows large. 
In summary, from the results in Table \ref{Table 3},  SIST  outperforms the four first-tier algorithms in terms of the running time, which implies the effectiveness of the proposed strategies.

\begin{figure}
	\centering
	\subfigure{\includegraphics[width=0.32\linewidth]{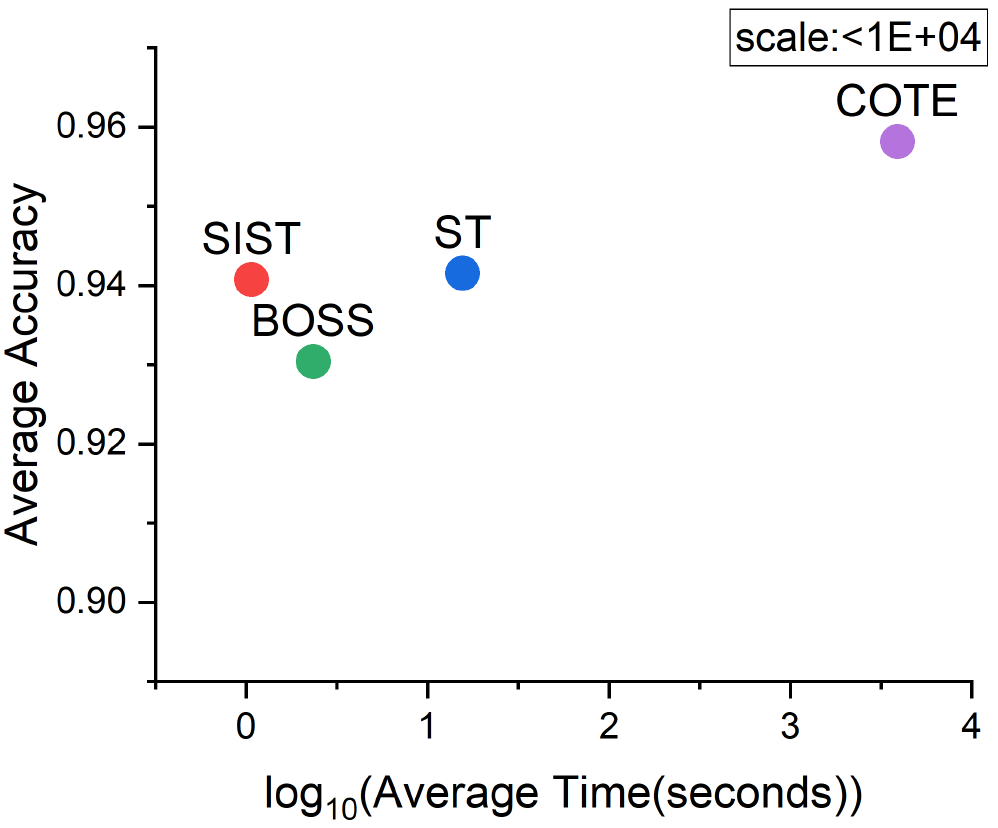} }
	\subfigure{\includegraphics[width=0.33\linewidth]{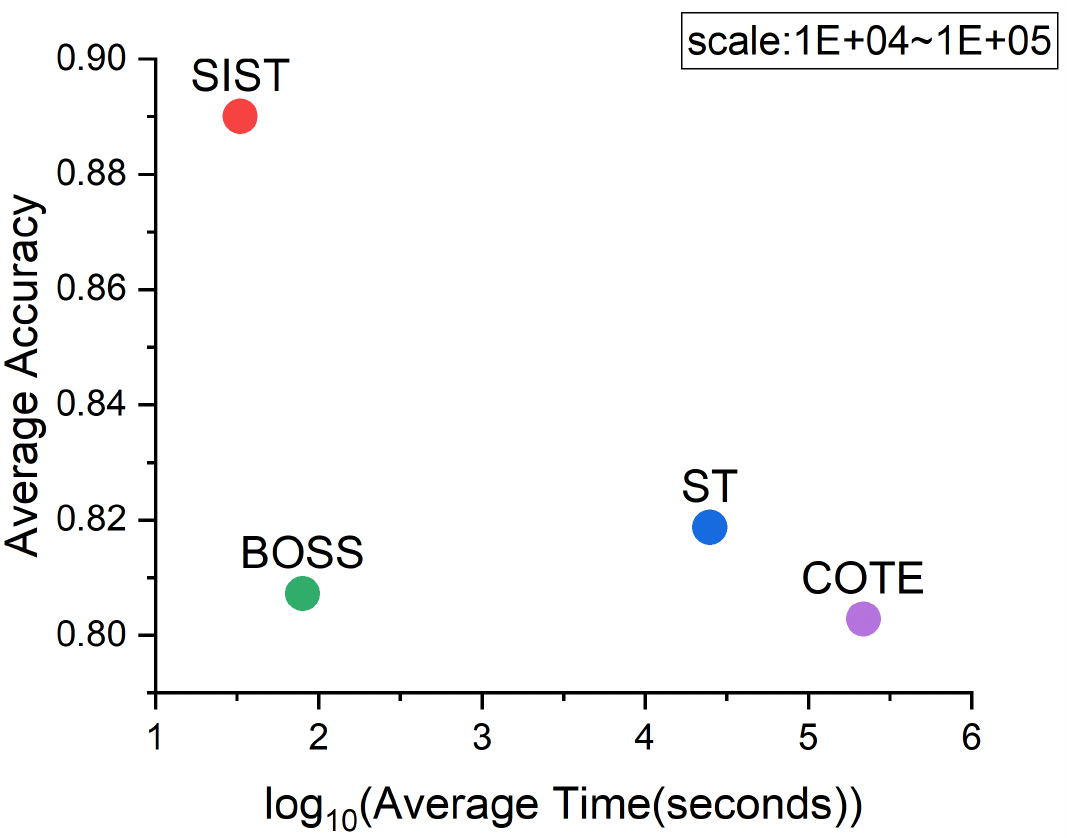}}
	\subfigure{\includegraphics[width=0.32\linewidth]{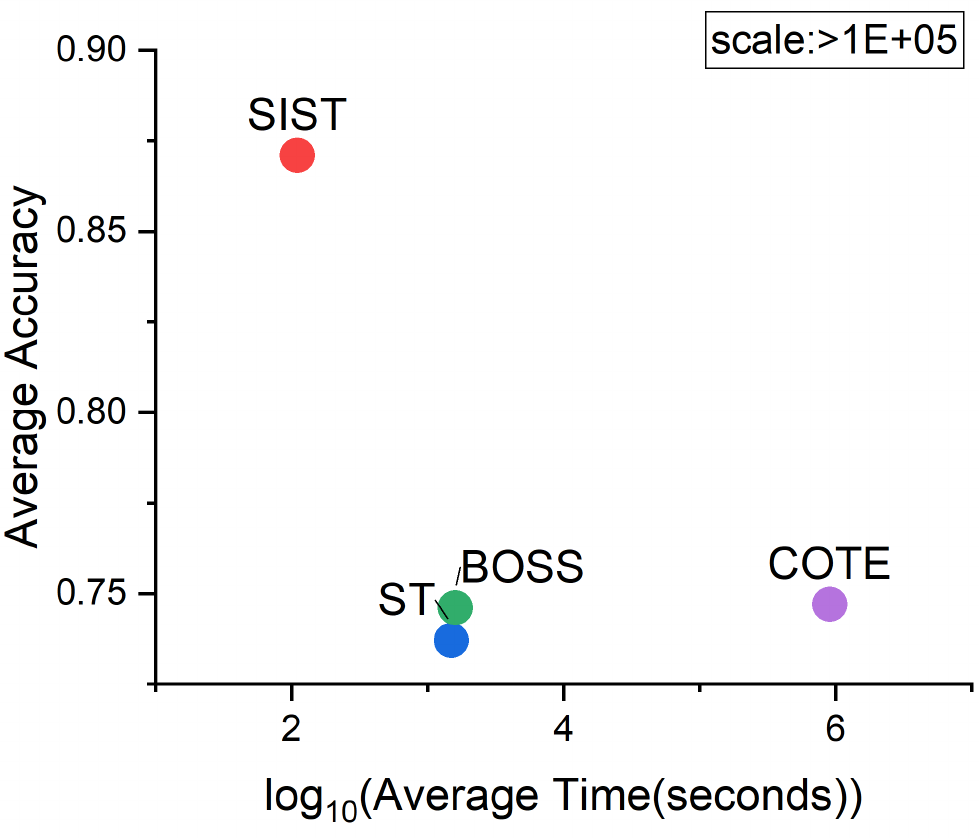}}
	\caption{Time-Accuracy comparison on data sets of different scale (scale is the product of train set size and data length)}
	\label{Fig 10}
\end{figure}

Next, a further comparison about the efficiency is conducted. By defining the product of the train set size and the data length as the scale of a data set, the experimental data sets are sorted into three categories according to their scales. Then the average accuracy and the average training time of these four compared algorithms are depicted in the coordinate system. Fig. \ref{Fig 10} illustrates the results where three kinds of data sets correspond to three graphs. The vertical axis represents the average accuracy and the horizontal axis represents the average training time. Note that the average training time is dealt with a logarithmic function basing $10$ since the difference of the average training time is too big to depict them in a coordinate system. Hence, the horizontal distance is not uniform or even. The real time difference increases exponentially as the horizontal distance grows linearly in the coordinate system.

Firstly, the average accuracy of all the compared algorithms declines as the scale of data sets inclines. Secondly, no matter what the scales of data sets are,  SIST  remains the lowest average training time while the COTE algorithm has the highest one. Thirdly, on the data sets with a scale more than $1E+04$,  SIST  achieves the best average accuracy as well as the lowest average training time. Fourthly, as the scale of data sets increases, it is increasingly obvious of both the time advantage and the accuracy advantage of  SIST. Fifthly, there is no algorithm dominating  SIST  in  Fig. \ref{Fig 10}.

According to those discoveries in Table \ref{Table 3} and Fig. \ref{Fig 10},  SIST  indeed possesses an efficiency advantage on binary TSC problems. 
The superiority of  SIST  are supported by the two proposed strategies.

%\begin{table*}[!t]
%\caption{Rank Comparison}
%\label{Table 4}
%\centering
%\begin{tabular}{|c||cccccccc||c|}
%\hline
%Attribute & DTWF & ST & BOSS & TSF & TSBF & LPS & EE & COTE & SIST\\
%\hline
%Rank Mean & 6.000 & 3.267 & 3.933 & 5.933 & 6.533 & 7.533 & 6.267 & 2.933 & $\bm{2.267}$\\
%\hline
%Difference & 3.733 & 1.000 & 1.666 & 3.666 & 4.266 & 5.266 & 4.000 & 0.666 & 0\\
%\hline
%Significantly Better by Friedman Test & Yes & No & No & Yes & Yes & Yes & Yes & No &  \\
%\hline
%\end{tabular}
%\end{table*}

\begin{table*}[!t]
	\renewcommand{\arraystretch}{1.8}
	\caption{The performance of SIST compared with LS, EL, and FS(CPU : Intel(R) Core(TM) i7-6500U CPU @ 2.50GHz ; RAM : 8GB ; Operating System : Windows 10 x64).
	DPOC: DistalPhalanxOutlineCorrect. 
	ECGFD: ECGFiveDays.
	IPD: ItalyPowerDemand.
	MPOC: MiddlePhalanxOutlineCorrect.
	PPOC: ProximalPhalanxOutlineCorrect
	SAIBORS1: SonyAIBORobotSurface1.
	SAIBORS2: SonyAIBORobotSurface2.
	TLECG: TwoLeadECG.
	}
	\label{ls-el-fs}
	\centering
	\begin{tabular}{|c|c|c|c|c|}
		\hline
		Accuracy/Time(s) & LS & EL & FS & SIST\\
		\hline
		Coffee & 1.000(1)/77.6(4) & 1.000(1)/0.4(1) & 0.964(4)/5.3(3) & 1.000(1)/0.8(2)\\
		\hline
		DPOC & 0.741(2)/305.9(4) & 0.713(4)/0.4(1) & 0.728(3)/18.6(2) & 0.830(1)/88.9(3)\\
		\hline
		Earthquakes & 0.748(2)/4.7E+03(4) & 0.711(4)/3.6(1) & 0.712(3)/1.1E+03(3) & 0.871(1)/61.3(2) \\
		\hline
		ECG200 & 0.850(2)/57.9(4) & 0.84(3)/0.1(1) & 0.75(4)/2.3(3) & 0.86(1)/1.5(2) \\
		\hline
		ECGFD & 1.000(1)/16.5(4) & 0.920(4)/0.1(1) & 0.995(2)/0.7(3) & 0.978(3)/0.3(2) \\
		\hline
		GunPoint & 1.000(1)/49.3(4) & 0.967(2)/0.1(1) & 0.94(4)/1.4(3) & 0.967(2)/0.9(2) \\
		\hline
		Ham & 0.686(2)/1.0E+03(4) & 0.600(4)/1.1(1) & 0.667(3)/157.8(3) & 0.838(1)/6.4(2) \\
		\hline
		Herring & 0.609(3)/763.1(4) & 0.641(2)/4.2(1) & 0.609(3)/91.2(3) & 0.875(1)/5.4(2) \\
		\hline
	    IPD & 0.962(2)/5.2(4) & 0.946(3)/0.1(1) & 0.906(4)/0.1(1) & 0.978(1)/0.1(1) \\
		\hline
		MPOC & 0.780(2)/323.3(4) & 0.750(3)/0.7(1) & 0.663(4)/15.1(2) & 0.897(1)/34.1(3) \\
		\hline
		MoteStrain & 0.913(1)/6.5(4) & 0.888(3)/0.1(1) & 0.798(4)/0.2(3) & 0.889(2)/0.1(1) \\
		\hline
		PPOC & 0.767(3)/303.1(4) & 0.738(4)/2.2(1) & 0.838(2)/11.7(2) & 0.900(1)/50.0(3) \\
		\hline
		SAIBORS1 & 0.952(1)/4.7(4) & 0.929(2)/0.1(1) & 0.686(4)/0.2(3) & 0.839(3)/0.1(1) \\
		\hline
		SAIBORS2 & 0.890(1)/6.4(4) & 0.749(4)/0.1(1) & 0.790(3)/0.2(2) & 0.877(2)/0.2(2) \\
		\hline
		Strawberry & 0.884(4)/2.2E+03(4) & 0.941(2)/0.7(1) & 0.908(3)/77.8(2) & 0.959(1)/279.7(3) \\
		\hline
		TLECG & 1.000(1)/7.2(4) & 0.990(2)/0.1(1) & 0.946(4)/0.2(3) & 0.982(3)/0.1(1) \\
		\hline
		Wine & 0.500(4)/144.0(4) & 0.537(3)/0.3(1) & 0.778(2)/4.7(3) & 1.000(1)/0.6(2) \\
		\hline
	\end{tabular}
\end{table*}

Besides  the state-of-the-art algorithms, 
SIST  is compared with other shapelet-based time reduction algorithms in the research field of TSC. 
They include LS (Learning Time-series Shapelets) \cite{Grabocka2014LearningTS}, 
EL (Efficient Learning of Time-series Shapelets)\cite{Hou2016EfficientLO}, 
and FS (Fast Shapelets) \cite{Keogh2013FastSA}. 
For a fair comparison, we keep the default hyperparameter settings in their papers and use the code they offer to the public. 
The result is shown in Table \ref{ls-el-fs}. 
From the table, we observe that  EL is the most efficient algorithm and  SIST is the most accurate algorithm. 
LS has a reasonable accuracy but a comparatively poor efficiency, and the FS reverses the situation in the LS. 
On some data sets,  SIST does not achieve the best accuracy, 
but it still possesses a fairly well accuracy close to the best accuracy. 
Though  SIST is not as efficient as the EL, the running time of it is also acceptable in most cases. 
EL  is based on a gradient method. 
It formulates the shapelet extraction problem as an optimization problem, 
but the weakness is that the optimization problem is not convex due to the nonlinear constraint. 
Therefore, it may have the overfitting problem or even the divergent problem, which leads to unsteady classification performance. 
For instance, on the data set 'Wine', both two gradient methods diverge, the performance of both  EL and  LS is not good. 
While the  running time of  SIST is slightly longer than EL, but its accuracy is  much better than EL.
From a holistic view,  SIST  has the steadiest and possibly the most accurate classification performance in the experiments.
%and its running time is also good  in the four compared algorithms. 

\subsection{Hyperparameter Analysis}
There are four types of hyperparameters in  SIST. 
To study their respective influence on the result, 
we fix  other three types of hyperparameters when studying one type of them. 
For instance, if the hyperparameter `shapelet length’ is studied, 
then there are $2 \times 4 \times 10=80$ possible combinations of  other three types of hyperparameters \ref{Table 0}.
For each possible combination, we consider all the possible choices of the `shapelet length’ 
and get the classification accuracy of each choice with the fixed three other hyperparameters.
Then, a group of classification accuracy is obtained on each data set for each combination of  other three hyperparameters, 
and  the standard deviation of the group of classification accuracy is calculated. 
Therefore, for each data set and each possible combination of  other three hyperparameters, 
a standard deviation is reported to show the sensitivity of  SIST on that data set to the hyperparameter `shapelet length’. 
An overall result is depicted in the above line charts. 
In the line chart, each polyline represents a data set, 
and each point in the polyline is the sensitive extent of  SIST to the relative hyperparameter 
on the data set under a specific setting of  other three hyperparameters. 
The sensitivity is the standard deviation calculated according to the rule described above. 

\begin{figure*}[htp]
	\centering
	\subfigure[overlap deletion]{\includegraphics[width=0.49\linewidth]{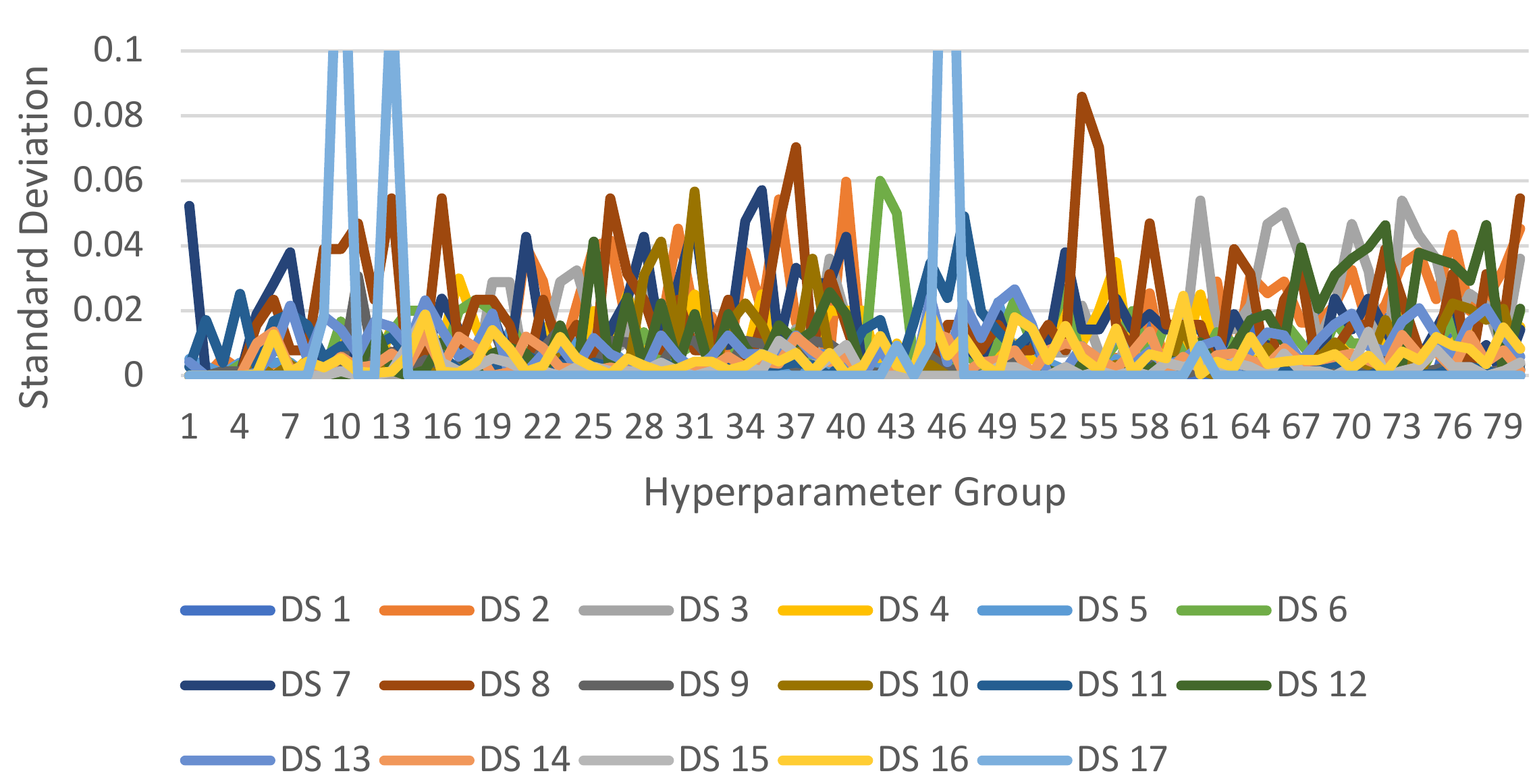} }
	\subfigure[shapelet length]{\includegraphics[width=0.49\linewidth]{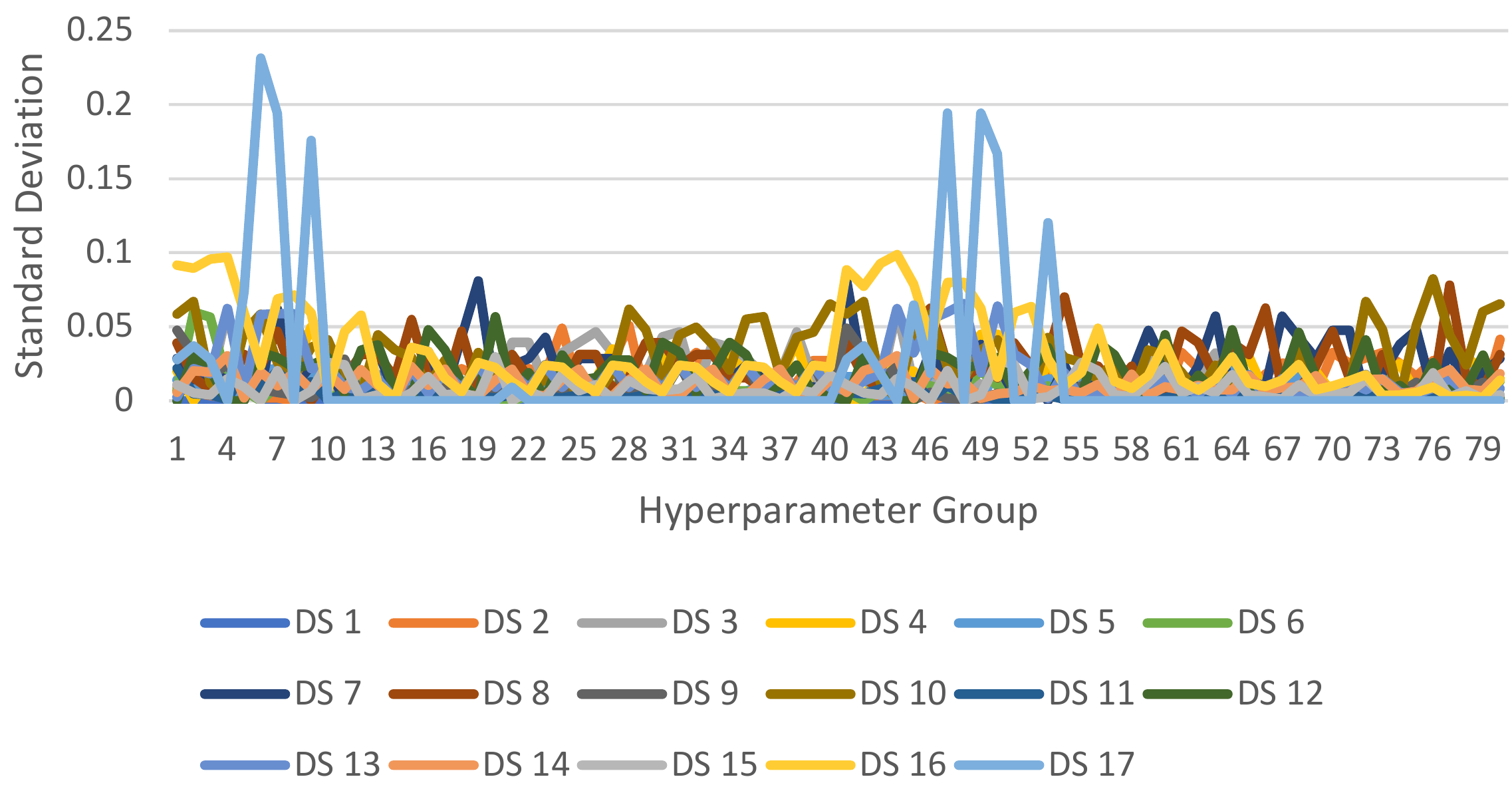}}
	\subfigure[relax factor]{\includegraphics[width=0.49\linewidth]{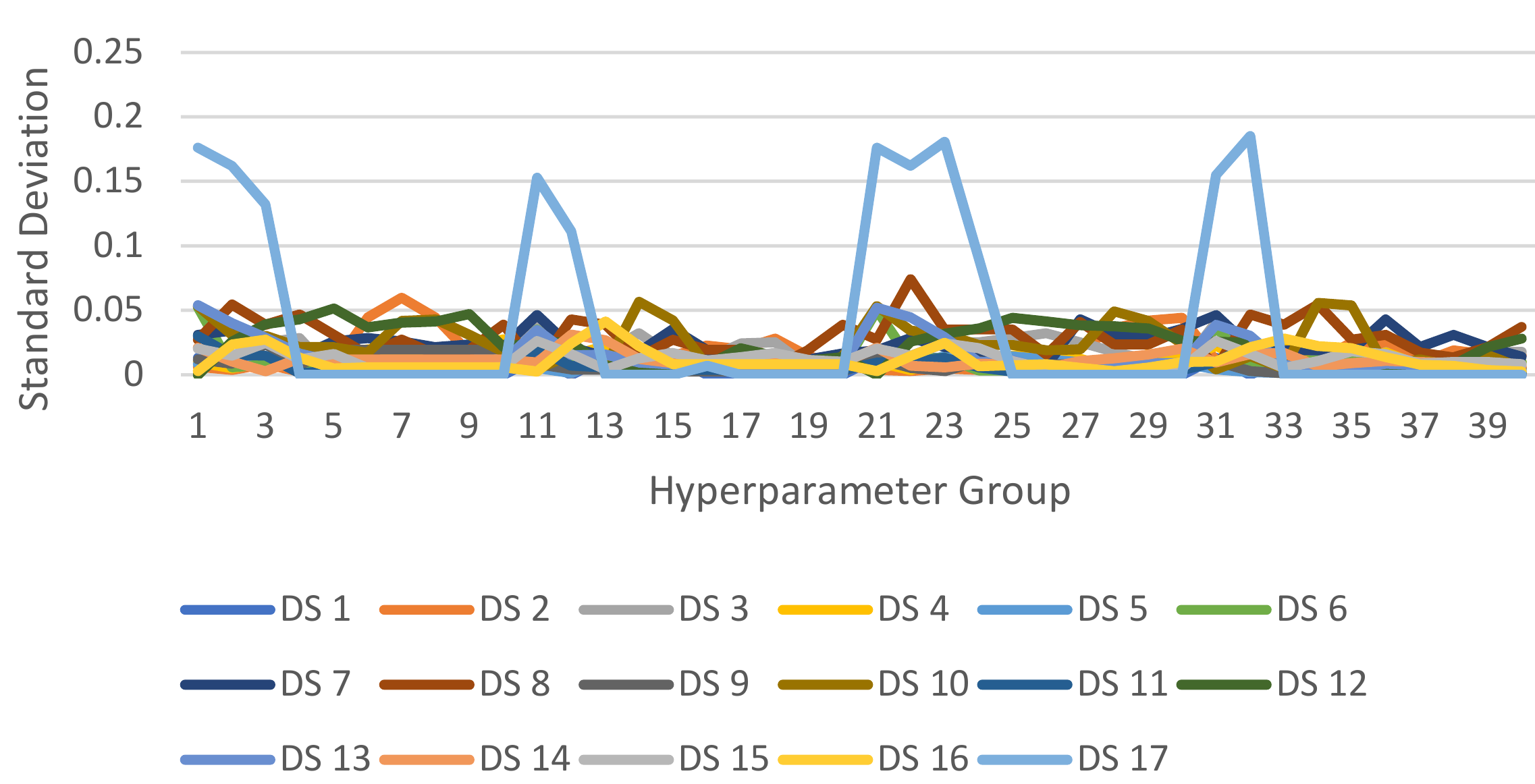}}
	\subfigure[shapelet num]{\includegraphics[width=0.49\linewidth]{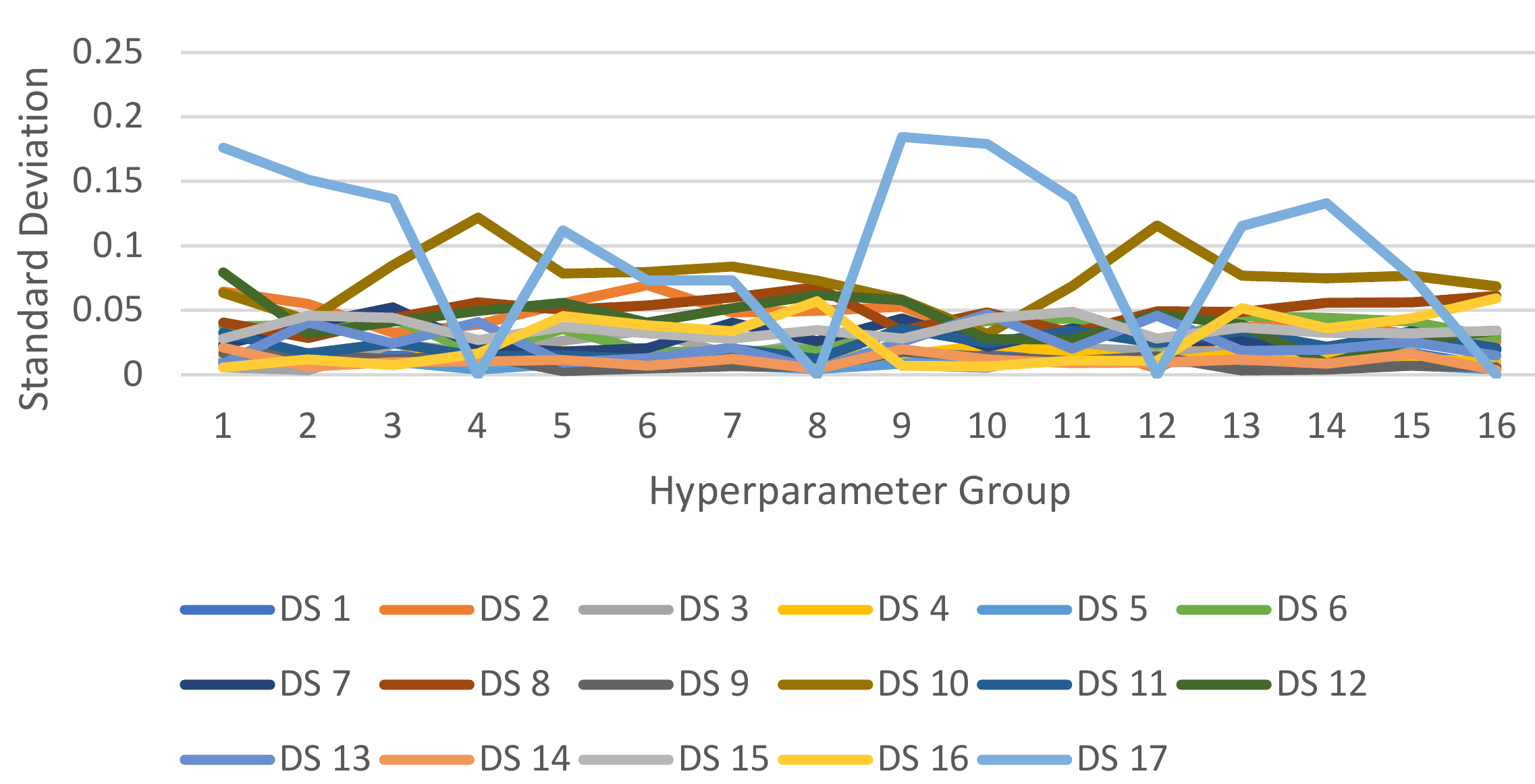}}
	\caption{Hyperparameter analysis}
		\label{overlap}
\end{figure*}

At first glance, the overall results are haphazard. But in fact,  some observations can be found. 
Firstly, the sensitivity of  SIST to hyperparameters differs on different data sets. 
Secondly,  the sensitivity of  SIST to one type of hyperparameter is strongly affected by the setting of  other three types of hyperparameters. 
Thirdly, the rank of the sensitivity of  SIST to the four types of hyperparameters may also differ on different data sets. 
Note that the third discovery is different from the first discovery. 
The first result indicates that the property of data sets has influences on the sensitivity of  SIST to the four types of hyperparameters, 
and the third result indicates that the influences of data sets on the sensitivity of  SIST to hyperparameters are not even or consistent on the four types of hyperparameters.

From the above observation, 
the sensitivity of  SIST to a single type of hyperparameter is affected by 
both the data set and the setting of  other three types of hyperparameters. 
Now, we analyze the influences of the hyperparameter from a higher perspective. 
We try discarding the factor of the data set and the other types of hyperparameters 
and exploring the overall sensitivity of  SIST to each type of hyperparameters. 
For each type of hyperparameter, we count the relative frequency of the standard deviation largger than some threshold in Figure \ref{overlap}. 
For instance, for the hyperparameter ‘shapelet length’, 
there are $17$ data sets and $80$ groups of settings of  other three types of hyperparameters. 
Therefore, there are $17 \times 80=1360$ standard deviation points. 
Then, for each threshold, we count the proportion of the points larger than the threshold in all the $1360$ points. 
Finally, we plot each proportion together in a polyline. The result is shown in the following figure. 

\begin{figure}[htp]
	\centering
	\includegraphics[width=0.8\linewidth]{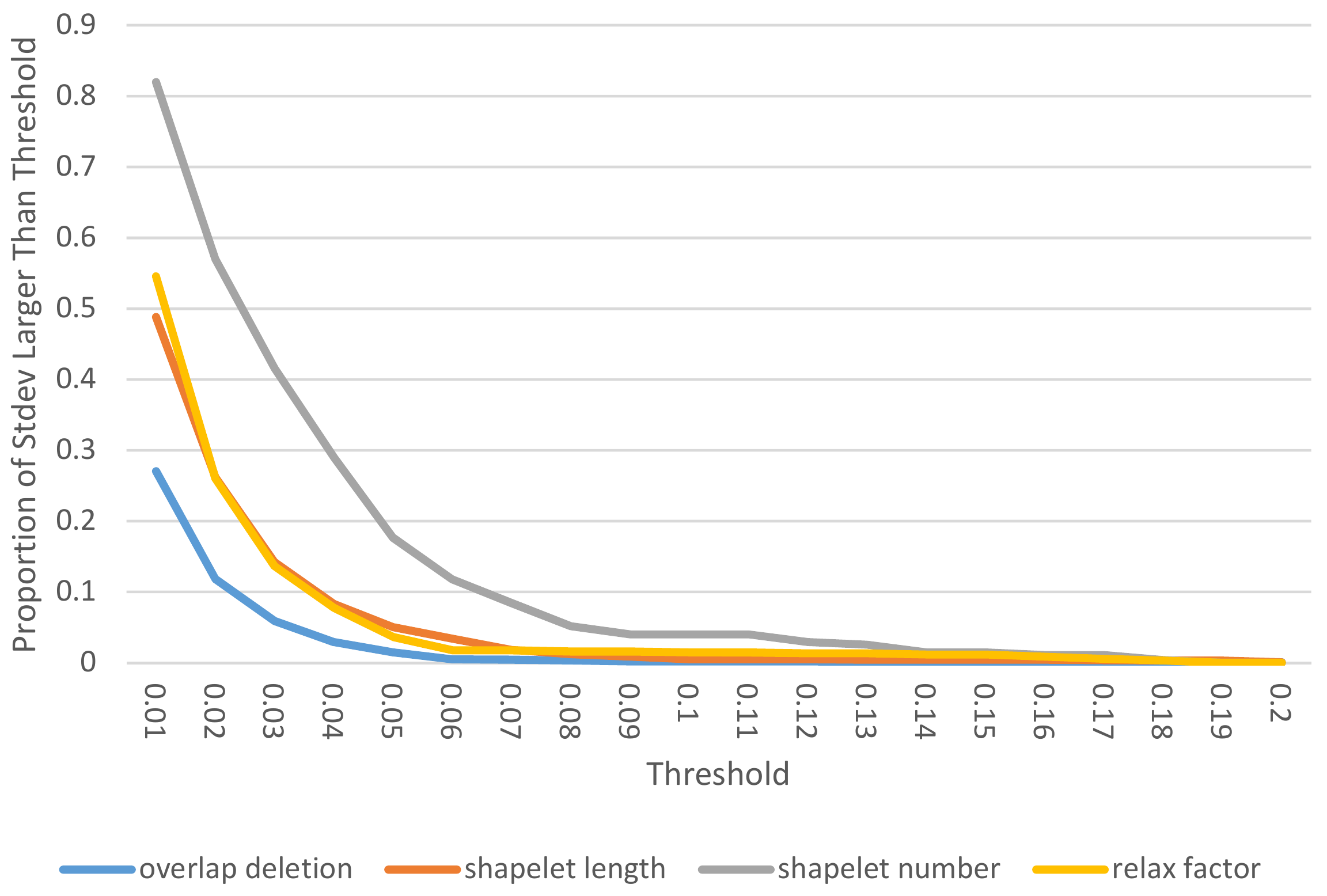}
	\caption{The sensitivity of the SIST to the four types of hyperparameters}
	\label{threshold}
\end{figure}

The result in Figure \ref{threshold}  indicates that  SIST is the most sensitive to `shapelet number’ in the four types of hyperparameters, 
and the most insensitive one is `overlap deletion’. 
For one type of hyperparameter, we can roughly regard the standard deviation 
as an average fluctuation range of the accuracy of  SIST on different choices of the type of hyperparameter. 
Then some information can be derived from some key points in the figure. 
The data located at $0.05$ on the Threshold axis show that there are nearly less than $5\%$ cases 
with an accuracy fluctuation range larger than $5\%$ 
when the changing hyperparameter is `overlap deletion’, 
`shapelet length’, or `relax factor’. 
The data located at $0.1$ on the Threshold axis show that 
there are almost no cases with an accuracy fluctuation range larger than $10\%$ 
when the changing hyperparameter is `overlap deletion’, `shapelet length’, or `relax factor’. 
Therefore, the influence of the hyperparameter `shapelet number’ on  SIST 
may be stronger than  other three types of hyperparameters. 

\begin{figure*}[htp]
	\centering
	\subfigure[overlap deletion]{\includegraphics[width=0.49\linewidth]{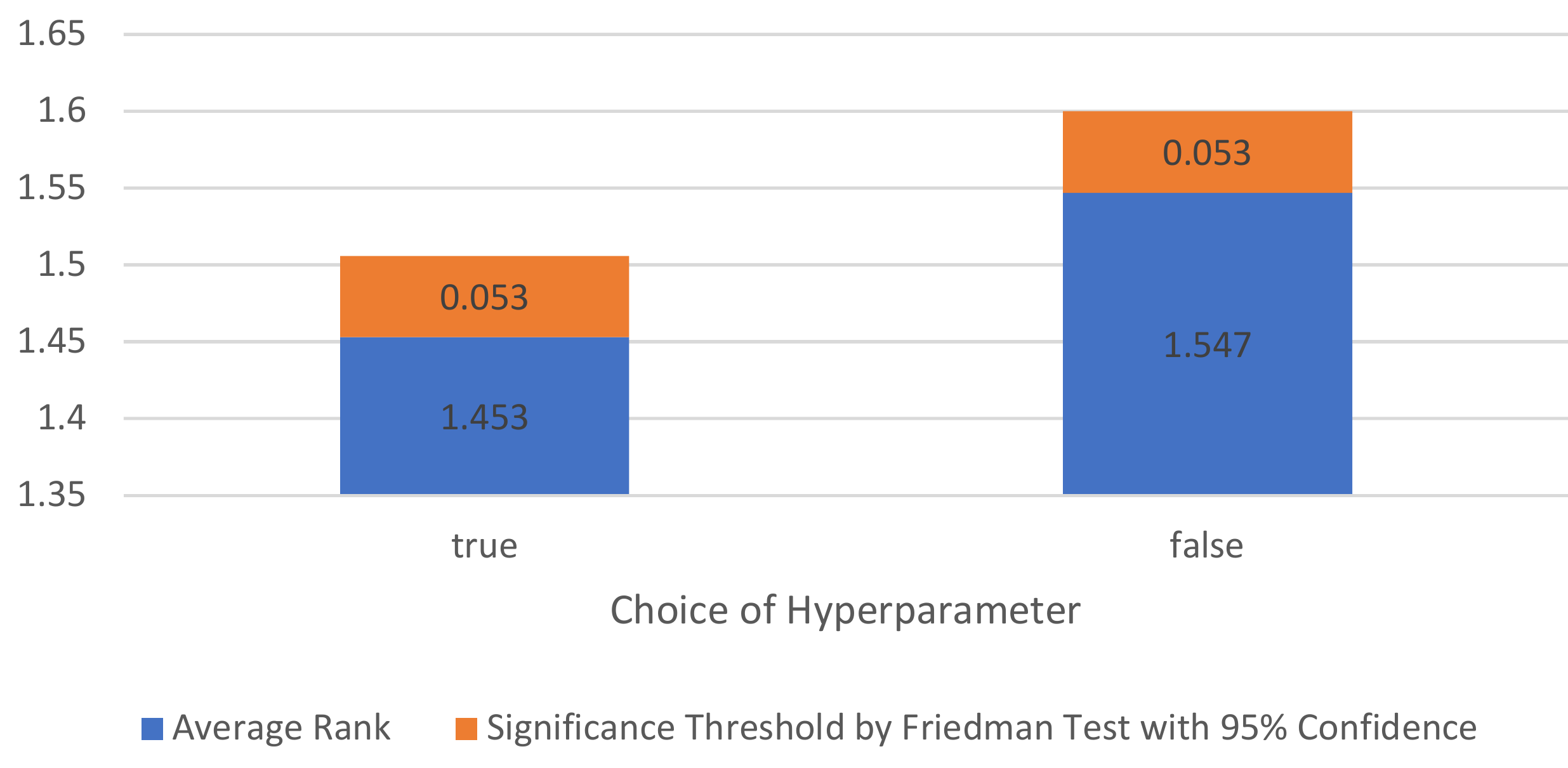} }
	\subfigure[shapelet length]{\includegraphics[width=0.49\linewidth]{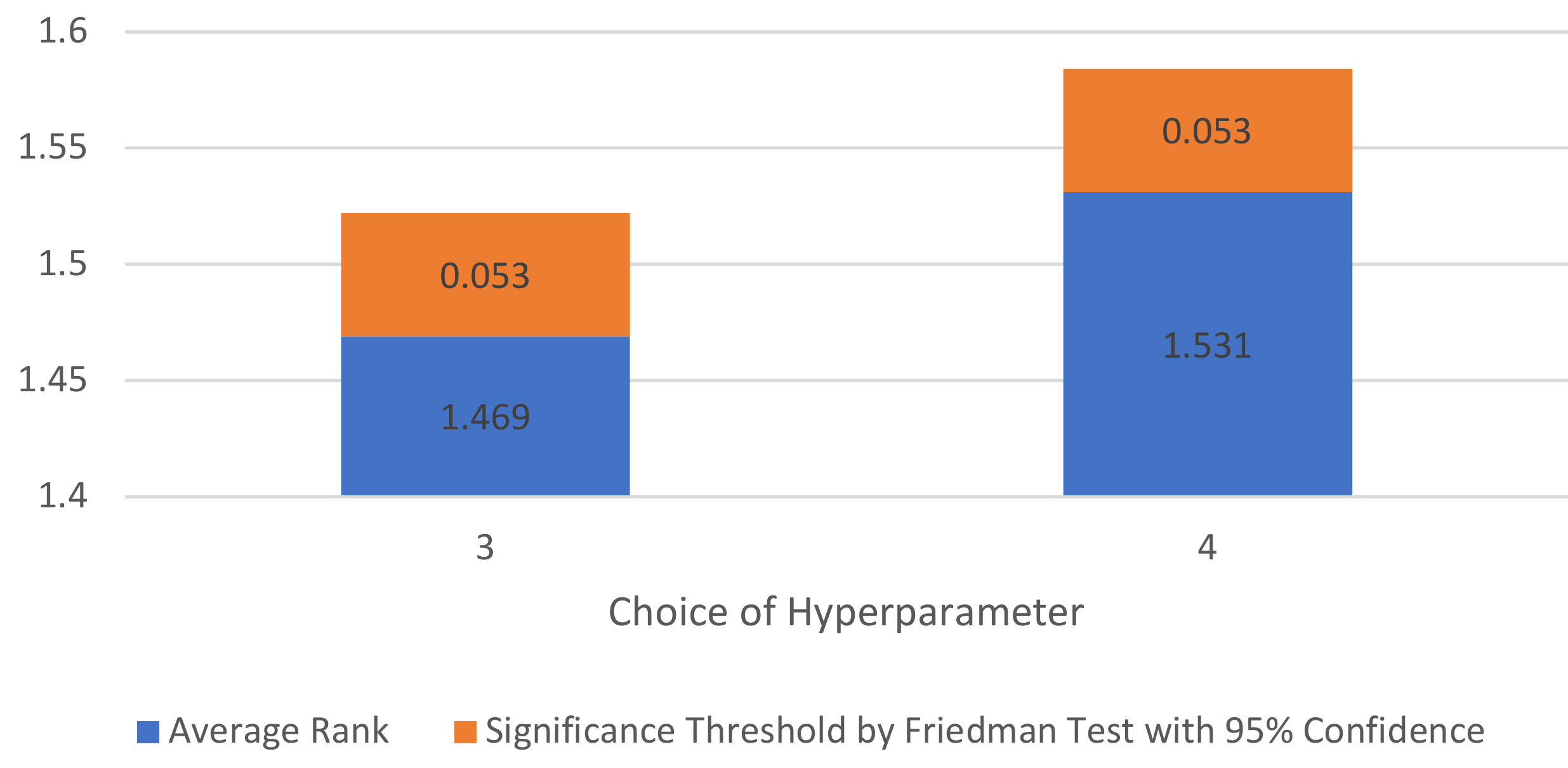}}
	\subfigure[relax factor]{\includegraphics[width=0.49\linewidth]{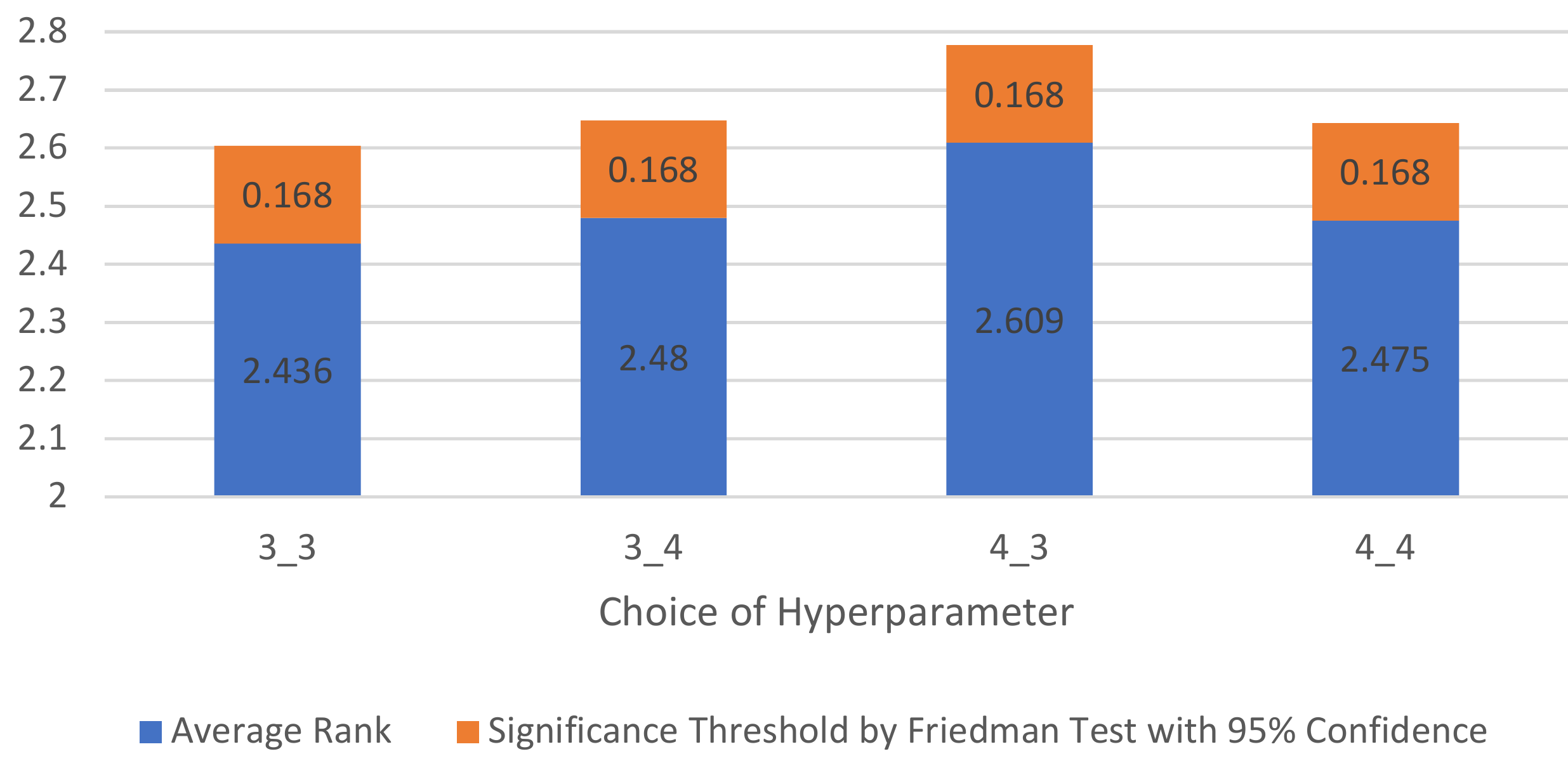}}
	\subfigure[shapelet num]{\includegraphics[width=0.49\linewidth]{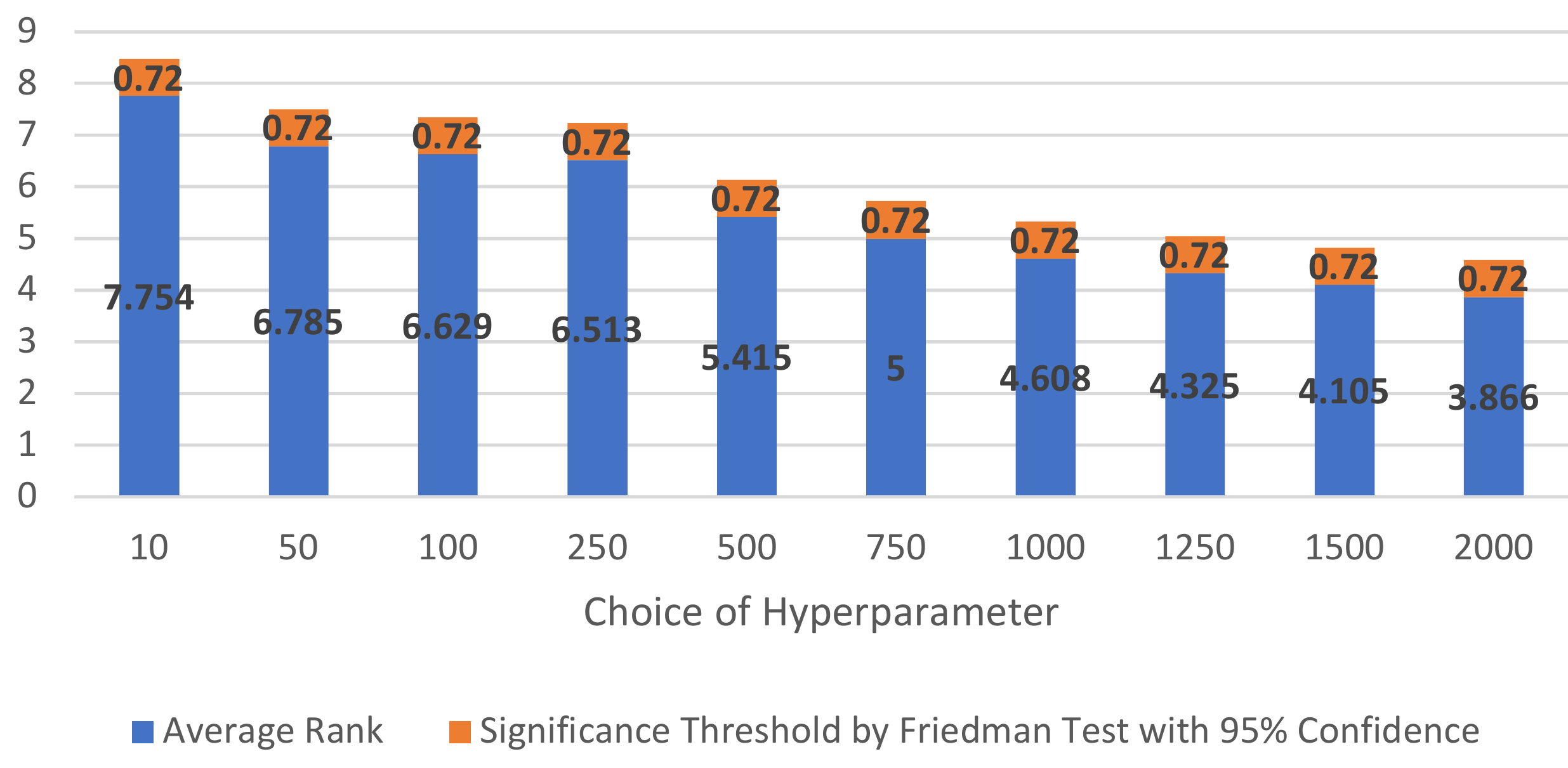}}
	\caption{The Friedman test for the preference of the choices of hyperparameters}
		\label{friedman-test}
\end{figure*}

Another observation is about the preference of  SIST to some choices of each type of hyperparameter. 
We carry out the comparison experiment on each choice for each type of hyperparameter. 
For instance, for the hyperparameter `shapelet number’, there are $17$ data sets and 
$2 \times 2 \times 4=16$ settings of  other three types of hyperparameters.
Therefore, 
there are $17 \times 16=272$ cases; For each case, there are $10$ choices of the `shapelet number’, 
and we can get the accuracy for each choice. 
By ranking these accuracies, we can further get a rank of each choice in each case. 
Finally, we calculate the average rank of each choice on all the $272$ cases and depict it in Figure \ref{friedman-test}. 
We check the significance of the preference of  SIST on some choices by the Friedman test with a $95\%$ confidence. 
Then we can check the preference of  SIST on any pair of choices of one type of hyperparameter 
by comparing the difference of their average ranks with the corresponding Friedman threshold. 
Some obvious results are concluded as follows. On the hyperparameter `overlap deletion’, 
 SIST prefers `deleting overlap shapelets’ to `not deleting overlap shapelets’; On the hyperparameter `shapelet length’, 
 SIST prefers `shapelet length $3$’ to `shapelet length $4$’; On the hyperparameter `relax factor’, 
 SIST prefers `$3$_$3$’, `$3$_$4$’, and `$4$_$4$’ to `$4$_$3$’ 
(here, `$a$_$b$’ means that setting left relax factor as $a$ and right relax factor as $b$ in the calculation of the Relaxed Fixed Distance.), 
and there is no significant preference among `$3$_$3$’, `$3$_$4$’, and `$4$_$4$’, 
and there is also no significant preference among `$3$_$4$’, `$4$_$3$’, and `$4$_$4$’. 
On the hyperparameter `shapelet number’, the situation is a little complex since there are $10$ choices of the hyperparameter, 
but one can easily check the preference of  SIST on any pair of two choices. 
A manifest result is that  SIST may prefer more shapelets in general. 
However, please note that these results are derived in an overall view, 
and when it comes to some specific data set and hyperparameter setting, the results may not hold. 

\subsection{Ablation Experiments}

\begin{figure}[htp]
	\centering
	\includegraphics[width=0.8\linewidth]{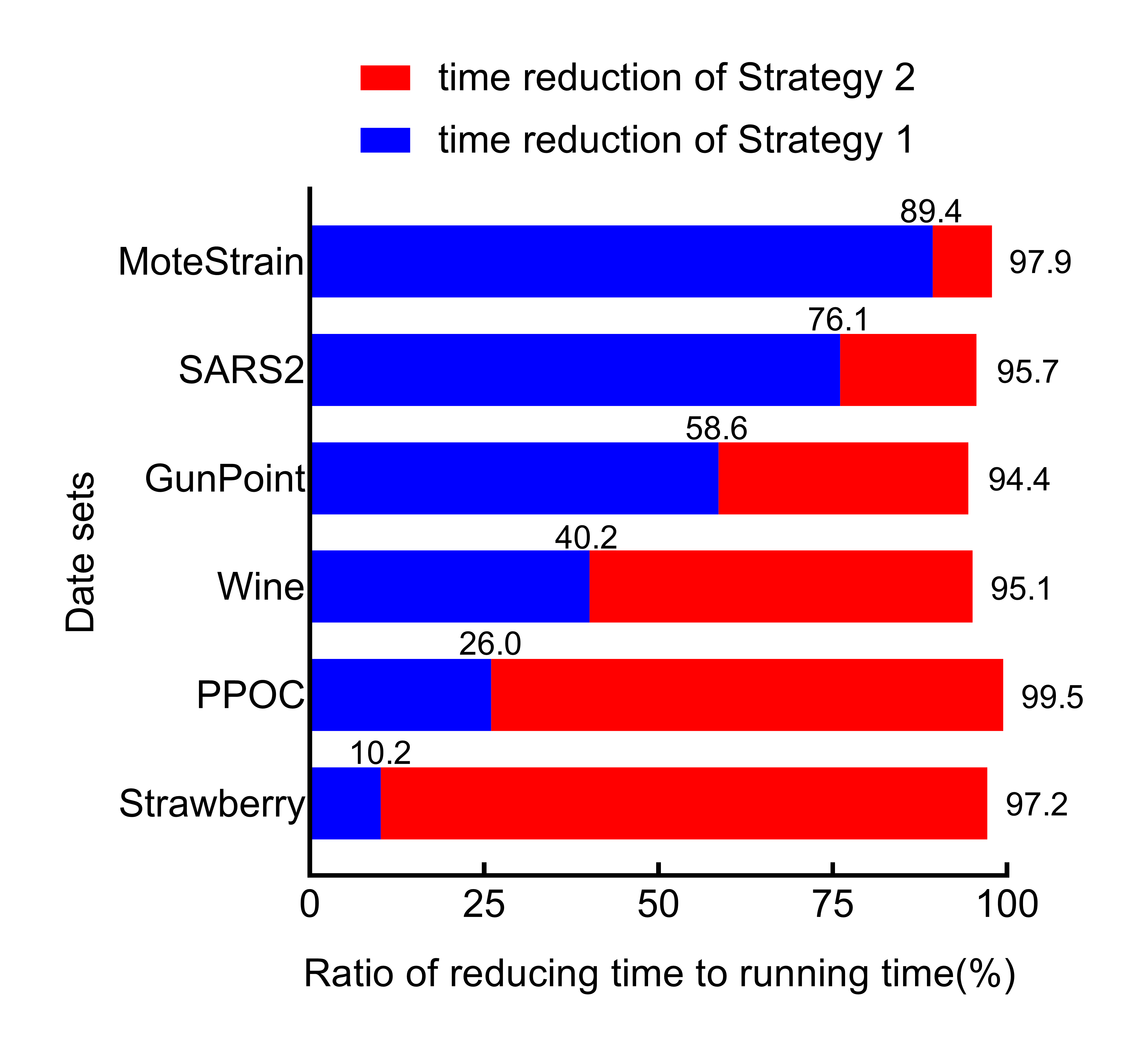}
	\caption{Ablation experiments. The amount of the total bar is the proportion of the total saving time in the running time of the ensemble ST.
		The blue bar corresponds to the proportion of the saving time in the first period while the red bar corresponds to the saving time in the second period.
		SARS2: SonyAIBORobotSurface2. PPOC: ProximalPhalanxOutlineCorrect.}
	\label{ablation}
\end{figure}

To further verify the contribution of each strategy on time reduction, 
ablation experiments are carried out on some data sets, and the results are depicted in Figure \ref{ablation}.
As previously mentioned, the two proposed strategies are used to reduce the time complexity of the ensemble ST algorithm. 
The first strategy embodies its efficiency in the shapelet extraction and feature space construction period. 
The second strategy takes effect in the classifier training period. 
And the two periods exactly compose the whole training process of the ST algorithm. 
Here, the running time of the ensemble ST on each strategy is compared with  SIST to indicate the effectiveness of the other strategy.

For different data sets, the absolute quantity of time reduction may differ in magnitude. 
Hence, the relative quantity of time reduction is used in the result for  more convenient comparison. 
In Figure \ref{ablation}, the amount of the total bar is the proportion of the total saving time in the running time of the ensemble ST. 
The blue bar corresponds to the proportion of the saving time in the first period while the red bar corresponds to the saving time in the second period. 
The data sets in this figure are arranged in the order of their scales. 
From the results, it is apparent that the two strategies are powerful in reducing the running time, 
since they together reduce $94.4\%\sim99.5\%$ running time of the ensemble ST on each data set. 
As the scale of the data set increases, Strategy 2 contributes more to time reduction than Strategy 1. 
Therefore, in the small scale data sets, Strategy 1 plays a major role in time reduction, whereas Strategy 2 is the key to time reduction in the large scale data sets. 
And when the data set is in a middle scale, the contributions of these two strategies are approximately the same.

\section{Conclusion and Future Work}
\label{section:conclusion}
In this paper, we have proposed an novel algorithm, short isometric shapelet transform (SIST), to solve the high time complexity problem of the ensemble  shapelet transform algorithm. 
Two strategies applied in SIST have been guaranteed by the theoretical evidences that shows SIST reduces the time complexity with the near-lossless accuracy.
Furthermore, the empirical experiments demonstrate the effectiveness of the proposed algorithm. 
Supported by the theoretical evidences, the two strategies adopted by SIST actually have generality in algorithms based on the shapelet transform. 
Hence, the future work is trying to apply these two strategies to more scenarios such as multi-class classification \cite{Lyu-2019-mpcvm}, semi-supervised learning and so forth. 
In the future, theoretical evidences will also be further explored to improve the theory  of these two strategies in SIST.

% BibTeX users please use one of
\bibliographystyle{spbasic}      % basic style, author-year citations
\bibliography{reference}

\end{document}